\documentclass[journal,twoside,web]{ieeecolor}
\usepackage{generic}
\usepackage{amsmath,amssymb,amsfonts}

\usepackage{amsthm}

\usepackage{textcomp}
\usepackage{bm}
\usepackage{cite}

\usepackage{multirow} 
\usepackage{makecell}
\usepackage{booktabs}
\usepackage{footnote}
\usepackage{threeparttable}

\usepackage{graphicx}
\usepackage[caption=false]{subfig}

\usepackage{caption2}
\usepackage{xcolor}
\usepackage{hyperref}

\usepackage{algorithm}
\usepackage{algorithmicx}
\usepackage{algpseudocode}

\newtheorem{definition}{Definition} 
\newtheorem{theorem}{Theorem}
\newtheorem{remark}{Remark}
\newtheorem{corollary}{Corollary}

\def\BibTeX{{\rm B\kern-.05em{\sc i\kern-.025em b}\kern-.08em
    T\kern-.1667em\lower.7ex\hbox{E}\kern-.125emX}}
\begin{document}
\title{Domain Adaptation for Industrial Time-series Forecasting via Causal Inference}
\author{Chao Min, Guoquan Wen, Jiangru Yuan, Jun Yi, Xing Guo}

\maketitle

\begin{abstract}
Industrial time-series, as a structural data responds to production process information,
can be utilized to perform data-driven decision-making 
for effective monitoring of industrial production process.
However, there are some challenges for time-series forecasting in industry,
including training cold-start and predicting few-shot caused by data shortage, 
and decision-confusing caused by unknown causal effect of stochasticaly treatment policy. 
To cope with the aforementioned problems, we propose a novel causal domain adaptation framework,
Causal Domain Adaptation (CDA) forecaster, for industrial processes. 
CDA leverages a Conditional Average Treatment Effect (CATE) 
from sufficiently historical time-series along with treatment policy (source) 
to improve the performance on the interested domain with limited data (target).
In particular, historical data consists of treatments (production measure), 
covariate (features) and outcome.
Firstly, we analyse the causality existing in industrial time-series along with treatments, and
thus ensure the shared causality over time.
Subsequently, we propose an answer-based attention mechanism to achieve domain-invariant 
representation by the shared causality in both domains.
Then, a novel domain-adaptation method is bulit to model the treatments and outcomes jointly
training on source and target domain.
The main insights are that 
our designed answer-based attention mechanism allows the target domain to leverage the casuality
existing in soure time-series even with different treatments, 
and our forecaster can predict the counterfactual outcome of industrial time-series
after an treatment, meaning a guidance in production process.
Compared with commonly baselines, our method on real-world abd synthetic oilfield datasets
demonstrates the effectiveness in  across-domain prediction and 
the practicality in guiding production process.
\end{abstract}


\section{Introduction}
\IEEEPARstart{I}{ndustrial} time-series forecasting, 
as a task reflecting industrial producing process over time, 
has drawn more attention along with the development of artificial intelligence, 
where CNN and RNN \cite{wang2020toward, liu2018time, canizo2019multi} are commonly used in forecasting 
due to their powerful ability in feature representation. 
There are various deformations of neural networks successfully applied to industry and engineering \cite{abiodun2018state}, 
e.g., mechanical system \cite{li2021causal}, petroleum industry \cite{al2012artificial}
organic chemistry \cite{fooshee2018deep} and biology \cite{wainberg2018deep}. 
In particularly, attention-based transformer-like models \cite{vaswani2017attention} 
have achieved state-of-the-art performance. 
A downside to these sophisticated models is their reliance on the large dataset 
with homogeneous time-series for training neural networks \cite{zhou2022time},

yet collecting and annotating the sufficient data are 
sometimes expensive and even prohibitive in industrial field. 
This data-efficiency challenge heavily impedes the adoption of deep learning to
a wide spectrum of industry \cite{bertolini2021machine}. 

An effective solution to this challenge is to explore the transferability in deep learning. 
The transferability extracts transferable representations 
from source tasks and then adapt it to improve the learning performance for target task 
\cite{bengio2012deep, lim2021temporal}. 
However, once trained, the deep learning models may not generalize well to 
a new domain of exogenous data since domain shift, that is, 
the distribution discrepancy between source and target domain \cite{wang2021bridging}.
To solve the domain shift issue in deep learning \cite{tzeng2017adversarial}, 
domain adaptation has been proposed to transfer the knowledge of source domain 
with sufficient data to the target domain with unlabeled or insufficiently labeled data for various tasks. 
Domain adaptation attempts to capture domain-invariant representations from raw data 
by aligning features across source and target domains \cite{niu2020decade}. 
In their seminal works, domain adaptation method based on neural network,  
e.g., DANN \cite{ganin2016domain}, MMDA \cite{rahman2020minimum}
is a popular methodology in exploring transferability \cite{zhang2019bridging}.
In the light of success in time-series task,
the related works \cite{ragab2023adatime, wang2023sensor} capture the domain-invariance 
by confusing the domain discriminator, and then perform the adversarially training.
These mechanism ensure they can effectivly distinguish representations from different domains, 
and outperform the metric-based neural network and domain adaptation method in forecasting task \cite{wilson2020survey}. 

In particular, industrial time-series (outcomes) always change over time and policy (treatments),
where treatments is a direct cause for outcomes.
In oil and gas industry, for example, gas production is positively affected by treatments 
such as injection and padding measure \cite{kan2020extended, bilgili2016did}. 
When the sequences of outcomes and treatments are recorded as a time-series, 
a policy dominates what treatment to take and when for production activity,
e.g., improving production and cutting down on accidents.
These methods accommodate the sequential causality of industrial time-series
in anomaly detection \cite{zhang2023spatial, cai2017bayesian, wang2023causal}
and outcome prediction \cite{li2021dtdr, li2023transferable}.
However, they do not directly reveal the latent causality existing in domain shift 
across the observed distribution of treatments and outcomes \cite{correa2022counterfactual}.
Since the policy itself is generally recorded in industrial production process,
we wan to infer a causal model of recurrent treatments from sequential treatment-outcome data
by utilizing domain adaptation and causal inference \cite{imbens2015causal, magliacane2018domain}.
In causal model, there are some consequences should be assessed for industrial production process, 
including, What is the causal effect of a given policy?
What will be the effect of a change to the treatment policy?
What would be happen if the production process had followed a different treatment policy?
These questions correspond to observational, interventional, and counterfactual queries. 
Answering them is crucial, especially in time-series forecasting along with domain shift, 
as well as a policy decision.
Specifically, answering the interventional and counterfactual queries requires modeling policy intervention
based on counterfactual inferecnce, while answering the observational query requires time-series forecaster.

Moreover, industrial time-series forecasting under limited data is a typical few-shot problem \cite{jin2022domain, teshima2020few}, 
which can result in poor-forecasting.  
Few-shot forecasting occurs when there is limited historical matching information in source domain. 
We can see that these problems are common phenomenon in industrial field 
due to the expensive cost in collecting and annotating data.

Based on above observations, we can know the advancement of industrial time-series forecasting is 
challenged by several critical bottlenecks, including
\begin{itemize}
    \item \textbf{Difficulty in collecting and annotating the sufficient data for training.}
    In industrial producing process, collecting and annotating the sufficient time-series data 
    for training model are expensive and even prohibitive due to complex work conditions or 
    privacy agreement. The data scarcity difficulty will lead to the few-shot problems 
    in domain adaptation.
    \item \textbf{Lack of causal guarantee for time-series forecasting.}
    While there are fruitful theoretical analyses for industrial time-series forecasting,
    similar analyses meet substantial hurdles to be extended to causality existing 
    in industrial time-series with treatment policy. 
    Its latent causality can provide the intrinsic relationship among multiple and multivariate time-series. 
    Therefore, it is essential to achieve causal guarantee, 
    which is beneficial to improve forecasting performance and provide producing guidance.
    \item \textbf{Difficulty in charactering the domain-invariant from the causality.}
    Charactering the domain-invariant involving domain shift requires 
    the bulit model is time-sensitive in industrial time-sereis. 
    An effective way is to make the attention mechanism as an encoder, 
    aiming at constructing the domain-invariant representations over time. 
    However, the traditional attention mechanism reconstructs the representation by utilizng the correlation
    not the causality, resulting in the absence of causality in domain shift, 
    Thus, it leads to a difficulty in estimating the causal impact of policy.
\end{itemize}

To figure out these situations, we propose a causal domain adaptation (CDA) framework, 
a novel sequence-to-sequence forecaster for industrial time-series forecasting. 
CDA leverages the advances in domain adversarial training and causal inference to 
construct a industrial time-series forecaster, which attempts to perform causal modeling across 
treatment-outcome and policy sequence. 
Our main contributions are as follows,
\begin{itemize}
    \item \textbf{Treatment-invariant representations over time.}
    Distinct from attention mechanism, CDA utilizes counterfactual inference to construct the 
    domain-invariant representations over time, i.e., treatment-invariant information
    which can break the association between history-matching and 
    treatment alignment of time-series. 
    Thus, the latent causality ensure that even limited data
    can provide the sufficient temporal information for time-series modeling.
    For this, domain adversarial training is employed to trade-off between 
    the treatment-invariant representations and time-series forecasting. 
    We show that the representations remove the bias caused by policy and 
    can be reliably used for forecasting industrial time-series. 
    This contribution also provides the causal guarantee for time-series forecasting from the 
    treatment policy.
    \item \textbf{Counterfactual estimation for future producing activity.}
    To estimate counterfactual outcome for treatment policy, 
    we integrate an average causal effect estimation (CATE) as a part of a sequence-to-sequence architecture. 
    CATE is mainly utilized to perform the causal domain-invariant representation
    and calculate the causal effect of treatment policy impacting on outcome sequence, 
    Thus, CDA can answer the question, which policy is most effective for producing activity?
    We illustrate in Fig.??? the applicability of CDA in counterfactual estimation and 
    optimizing in oil and gas filed.   
\end{itemize}

In our experiments, we evaluate CDA in the realistic datasets with treatment policy in oil and gas field. 
We show that CDA achieves better performance in forecasting monthly oil production, 
but also in choosing the optimal treatment policy in improving oil production.

\section{Preliminaries}
Our work mainly builds on causal inference and domain adaptation.

\subsection{Casual Inference}
Casual inference is the process of drawing a conclusion about a causal connection 
based on the conditions of the occurrence of an effect \cite{imbens2015causal}. 
It is well known that “correlation does not imply causation.” 
The main difference between causality and correlation analysis is that 
former analyzes the response of the effect variable when the cause is changed. 
In this section, we briefly introduce causal graph, treatment intervention and counterfactual inference
for our method \cite{johansson2016learning}. 
And we also define a causal answer formally, while using SCM to emphasize the distinction 
between probability distributions induced by different types of causal answer.

\noindent \textbf{Causal Graph.} 
Let $\mathcal{M}:=(\mathbf{S},p(\epsilon))$ denote a structural causal model (SCM), 
which consists of structural assignments 
$\mathbf{S}=\{\mathbf{x}_{i}:=f_{i}(\epsilon_{i};\mathbf{pa}_{i})\}_{i=1}^{N}$ over structural model $f$
and joint distribution $p(\epsilon)=\prod_{i=1}^{N}p(\epsilon_{i})$ over noise variables $\epsilon$. 
In addition, $\mathbf{pa}_{i}$ is the set of parents of $\mathbf{x}_{i}$ (its direct causes). 
Then, a causal graph $\mathcal{G}$ is obtained by representing each variable $\mathbf{x}_{i}$ as node and 
draw edges to $\mathbf{x}_{i}$ from its parents $\mathbf{pa}(\mathbf{x}_{i})$. 
Formally, time-series $\mathbf{pa}(\mathbf{x}_{i})$ is said to be a temporal cause of $\mathbf{x}_{i}$ 
when the values of $\mathbf{pa}(\mathbf{x}_{i})$ provide statistically significant information 
about the future values of $\mathbf{x}_{i}$.

\noindent \textbf{Treatment Intervention.} 
Let $\mathbf{X},\mathbf{Y} \mathrm{~and~} \mathbf{Z}$ denote the covariate, outcome, 
and treatment variable, respectively. In causal inference, $\mathbf{Y}(\mathbf{X},\mathbf{z}_{i})$ 
is a random variable that represents the value of outcome $\mathbf{Y}$ 
under a treatment policy $\mathbf{z}_{i}$ and covariate $\mathbf{X}$. 
The treatment intervention on $\mathbf{Z}$ by $\mathbf{z}_{i}$ corresponds to the \textit{do-operator} 
$\mathrm{do}(\mathbf{Z}=\mathbf{z}_{i})$ in the SCM, as
$P\big(\mathbf{Y}(\mathbf{X},\mathbf{z}_{i})\big)=
P\big(\mathbf{Y}\big|\mathbf{X},\mathrm{do}(\mathbf{Z}=\mathbf{z}_{i})\big)$,
where $\mathrm{do}(\mathbf{Z}=\mathbf{z}_{i})$ means the treatment policy is forced to intervene to $\mathbf{z}_{i}$.

\noindent \textbf{Counterfactual Inference.}
In counterfactual queries, given a history of past events that have already occurred, 
one asks what past events would have instead occurred if certain intervention had been in place. 
The counterfactual can be formulized as $P\big(\mathbf{Y}(\mathbf{X},\mathbf{z}_{i}\big)\big|\mathbf{X},\mathbf{z}_{j})$, 
where $\mathbf{X},\mathbf{z}_{j}$ corresponds to the past events 
$\mathbf{Y}(\mathbf{X},\mathbf{Z}=\mathbf{z}_{j})$ 
and $\mathbf{Y}(\mathbf{X},\mathbf{z}_{i})$ corresponds to a query that 
if we make an intervention with $\mathrm{do}(\mathbf{Z}=\mathbf{z}_{i})$, 
what will happen to $\mathbf{Y}$ under the trained model $f$, 
i.e., past events $\mathbf{Y}(\mathbf{X},\mathbf{Z}=\mathbf{z}_{j})$.

The difference between treatment intervention $P\big(\mathbf{Y}\big|\mathbf{X},\mathrm{do}(\mathbf{Z}=\mathbf{z}_{i})\big)$
and counterfactuals inference $P\big(\mathbf{Y}(\mathbf{X},\mathbf{z}_{i}\big)\big|\mathbf{X},\mathbf{z}_{j})$ 
are highlighted by joint-distribution: 
(i) a treatment intervention requires access to the treatment distribution 
$\mathbf{z}_{i}\sim P(\mathbf{Z})$ 
and only use the noise prior $P(\epsilon)$; 
(ii) a counterfactual inference requires access to the counterfactual distribution 
$\mathbf{z}_{j}\sim P(\mathbf{Z})$ 
and the posterior distribution of noise $P(\epsilon|\mathbf{X}, \mathbf{z}_{i})$, 
which incorporates the knowledge of what already happened. Upon these knowledge, 
an example of causal graph in continuous time is shown in Fig.\ref{figs:pre} by causal inference.

\begin{figure}[!t]
    \centerline{\includegraphics[scale=0.4]{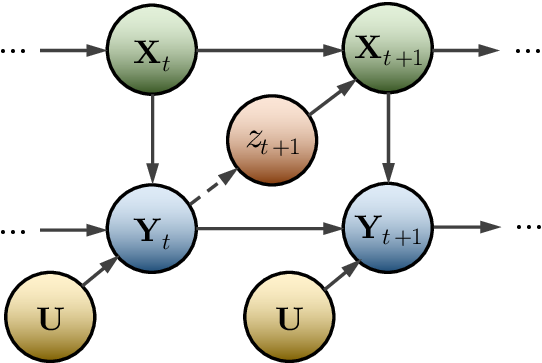}}
    \caption{The causal graph $\mathcal{G}$ in continuous time, 
    where $a\rightarrow b$ means $a$ is the direct cause of $b$. 
    And the treatment policy $z_{t+1}$ can be specified by human (dashed arrow) 
    or determined by previous outcome $\mathbf{Y}_{t}$, 
    $\mathbf{U}$ denote the static variables.}
    \label{figs:pre}     
\end{figure}

\subsection{Domain Adaptation under Causality}
In this section, we follow the definitions and notations of domain adaptation in \cite{pan2010domain}. 
Accordingly, a domain $\mathcal{D}$ is composed of a feature space $\mathcal{X}$ 
with marginal probability distribution $P(\mathbf{X})$ and a task $\mathcal{T}$ with label space $\mathcal{Y}$  
with conditional probability distribution $P(\mathbf{Y}|\mathbf{X})$, 
where $\mathbf{X}$ and $\mathbf{Y}$ are the covariate and target variable. 
By assuming the source domain $\mathcal{D}^{S}=\{{\mathcal X}^{S},P(\mathbf{X}^{S})\}$ 
with $\mathcal{T}^{S}=\{\mathcal{Y}^{S},P(\mathbf{Y}^{S}|\mathbf{X}^{S})\}$
and target domain $\mathcal{D}^{T}=\{{\mathcal X}^{T},P(\mathbf{X}^{T})\}$  
with $\mathcal{T}^{T}=\{\mathcal{Y}^{T},P(\mathbf{Y}^{T}|\mathbf{X}^{T})\}$, 
heterogenous domain adaptation can be stated that when the feature spaces of source and target domain 
satisfy $P(\mathbf{X}^{S})\neq P(\mathbf{X}^{T})$ with $\mathcal{X}^{S}=\mathcal{X}^{T}$ due to domain shift, 
the label spaces of source and target domain under conditional probability distribution holds that 
$P(\mathbf{Y}^{S}|\mathbf{X}^{S})\neq P(\mathbf{Y}^{T}|\mathbf{X}^{T})$ with $\mathcal{Y}^{S}\neq \mathcal{Y}^{T}$.

As aforementioned, this paper focuses on the domain adaptation in industrial time-series forecasting 
under treatment policy. As shown in Fig.\ref{figs:pre}, distinct from non-intervention mode, 
the domain shift in intervention mode between feature spaces 
$\mathbf{X}^{S}$ and $\mathbf{X}^{T}$ can be detected its causes, 
mainly concentrating on the corresponding treatment policy $\mathbf{z}_{i}, \mathbf{z}_{j}\in\mathbf{Z}$. 
These can be defined as follows,
\begin{equation}
    \begin{split}
        P(\mathbf{X}^{S}|\mathbf{Z}=\mathbf{z}_{i})&\neq P(\mathbf{X}^{T}|\mathbf{Z}=\mathbf{z}_{j})\\ 
        P(\mathbf{Y}^{S}|\mathbf{X}^{S},\mathbf{Z}=\mathbf{z}_{i})&\neq 
        P(\mathbf{Y}^{T}|\mathbf{X}^{T},\mathbf{Z}=\mathbf{z}_{j})
    \end{split}
    \label{eqs:casual-da}
\end{equation}
where $\mathbf{z}_{i} \mathrm{~and~} \mathbf{z}_{j}$ denote the treatment policy in source and target domain, 
namely the conditions.

\noindent  \textbf{Causal structure.} 
Intervened by treatment policy, source and target domains of industrial time-series
have different distribution but share the same causal structure over time, 
shown in the top of Fig.\ref{figs:pre}. 
Specifically, similar to the aforementioned notation of causal inference , 
the domain adaptation under causality has the similar probability representation, 
which can be category into two main groups depending on treatment policy or counterfactual inference. 
By utilizing cause $\mathbf{Z}$, namely treatment policy, 
the bridging theory of domain adaptation for treatment intervention can be defined as 
$P(\mathbf{X}^{S}|\mathbf{Z}=\mathbf{z}_{i})\neq P(\mathbf{X}^{T}|\operatorname{do}\big(\mathbf{Z}=\mathbf{z}_{j})\big)$, 
while for counterfactual inference can be defined as 
$P(\mathbf{Y}^{S}|\mathbf{X}^{S},\mathbf{Z}=\mathbf{z}_{i})\neq 
P(\mathbf{Y}^{T}(\mathbf{X}^{S},\mathbf{z}_{i})|\mathbf{X}^{T},
\operatorname{do}\big(\mathbf{Z}=\mathbf{z}_{j})\big)$

\section{Method}
In this section, the problem setup and notation used to study treatment-outcome and domain adaptation
is introduced. And then the bulit method will be proposed, that is, the causal domain adaptation (CDA).
\subsection{Problem formulation}
Considering an observational $N$ time-series dataset 
$\mathcal{D}=\left\{\left(\mathbf{x}_{t}^{(i)},\mathbf{y}_{t}^{(i)},\mathbf{z}_{t}^{(i)}\right)_{t=1}^{T^{(i)}}\right\}_{i=1}^{N}$, 
each observation consists of time-dependent covariates $\mathbf{x}_{t}^{(i)}\in\mathcal{X}_{t}$, 
treatments received from $\mathbf{z}_{t}^{(i)}\in\{z_{1},z_{2},\cdots,z_{K}\}$ and 
outcomes $\mathbf{y}_{t}^{(i)}\in\mathcal{Y}_{t}.$ for $T^{(i)}$ discrete timesteps. 
In time-series forecasting, given the history of the covariates $\mathbf{X}_{1:t}=(\mathbf{x}_{1},\mathbf{x}_{2},\cdots,\mathbf{x}_{t})$ 
and treatment assignments $\mathbf{Z}_{1:t}=(\mathbf{z}_{1},\mathbf{z}_{2},\cdots,\mathbf{z}_{t})$, 
we want to make $\tau$ predictions of outcomes $\mathbf{Y}_{t+1:t+\tau}=(\mathbf{y}_{t+1},\mathbf{y}_{t+2},\cdots,\mathbf{y}_{t+\tau})$ 
via future sequence and time-series generator $G$,
\begin{equation}
    \mathbf{Y}_{t+1:t+\tau}(\mathbf{X}_{t+1:t+\tau},\mathbf{Z}_{t+1:t+\tau})=G(\mathbf{X}_{1:t},\mathbf{Z}_{1:t},\mathbf{Y}_{1:t})
    \label{eqs:problem}
\end{equation}

According to the Eq.(\ref{eqs:problem}), the history of time-series can be used to train a supervised sequence 
generator by the probability formulation 
$P(\mathbf{Y}_{t+1:t+\tau}(\mathbf{X}_{t+1:t+\tau},\mathbf{Z}_{t+1:t+\tau})|\mathbf{X}_{1:t},\mathbf{Z}_{1:t},\mathbf{Y}_{1:t})$, 
Furthermore, in industrial time-series forecasting, the treatment policy $\mathbf{Z}_{t+1:t+\tau}$
can be specified in advance. Then, conditioned on the history $(\mathbf{X}_{1:t},\mathbf{Z}_{1:t},\mathbf{Y}_{1:t})$ 
and shared causality over time, the causal effect of the treatment $\mathbf{Z}_{t+1:t+\tau}$ 
on the outcome trajectory is fully mediated through sequential covariates 
$\mathbf{X}_{t+1:t+\tau}$ (\cite{hizli2023causal} Theorem 1.), 
\begin{equation}
    \begin{split}
        &P(\mathbf{Y}_{t+1:t+\tau}(\mathbf{X}_{t+1:t+\tau},\mathbf{Z}_{t+1:t+\tau})|\mathbf{X}_{1:t},\mathbf{Z}_{0:t-1},\mathbf{Y}_{1:t})\\ 
        =&P(\mathbf{Y}_{>t}(\mathbf{X}_{>t},\mathbf{Z}_{>t})|\mathbf{H}_{\leq t}) \\
        =&P(\mathbf{Y}_{>t}|\mathbf{H}_{\leq t},\mathbf{X}_{>t},\mathbf{Z}_{>t})
        P(\mathbf{X}_{>t}|\mathbf{H}_{\leq t},\mathbf{Z}_{>t})
        P(\mathbf{Z}_{>t}|\mathbf{H}_{\leq t})
    \end{split}
    \label{eqs:probability}
\end{equation}
where $\mathbf{H}_{\leq t}=\left[\mathbf{X}_{1:t},\mathbf{Z}_{1:t},\mathbf{Y}_{1:t}\right],
\mathbf{X}_{>t}=\mathbf{X}_{t+1:t+\tau},\mathbf{Z}_{>t}=\mathbf{Z}_{t+1:t+\tau}$
and $\mathbf{Y}_{>t}=\mathbf{Y}_{t+1:t+\tau}$. 
Accordingly, using the time-lag response in industrial time-series, 
both of which can be estimated with a specified time-series model,
\begin{equation}
    \underbrace{P(\mathbf{Y}_{>t}|\mathbf{H}_{\leq t},\mathbf{X}_{>t},\mathbf{Z}_{>t})}_{\text{Outcome Term}}
    \underbrace{P(\mathbf{X}_{>t}|\mathbf{H}_{\leq t},\mathbf{Z}_{>t})
    P(\mathbf{Z}_{>t}|\mathbf{H}_{\leq t})}_{\text{Treatment Term}}     
    \label{eqs:joint-model}
\end{equation}

Apparently, the identifiability result in Eq.(\ref{eqs:joint-model}) generalize the forecasting task to 
a time-dependent covariate forecasting module, namely \textbf{Treatment Term}, 
and another outcome forecasting module, namely \textbf{Outcome Term}.

\begin{figure}[!t]
    \centerline{\includegraphics[scale=0.4]{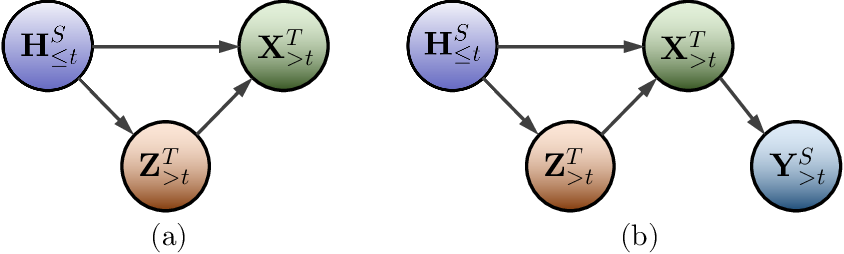}}
    \caption{The shared causality for both both domains in treatment term and outcome term.
    (a). The causality in treatment-policy term; (b). The causality in outcome term.}
    \label{figs:causality}     
\end{figure}

\begin{figure*}[!t]
    \centerline{\includegraphics[scale=0.4]{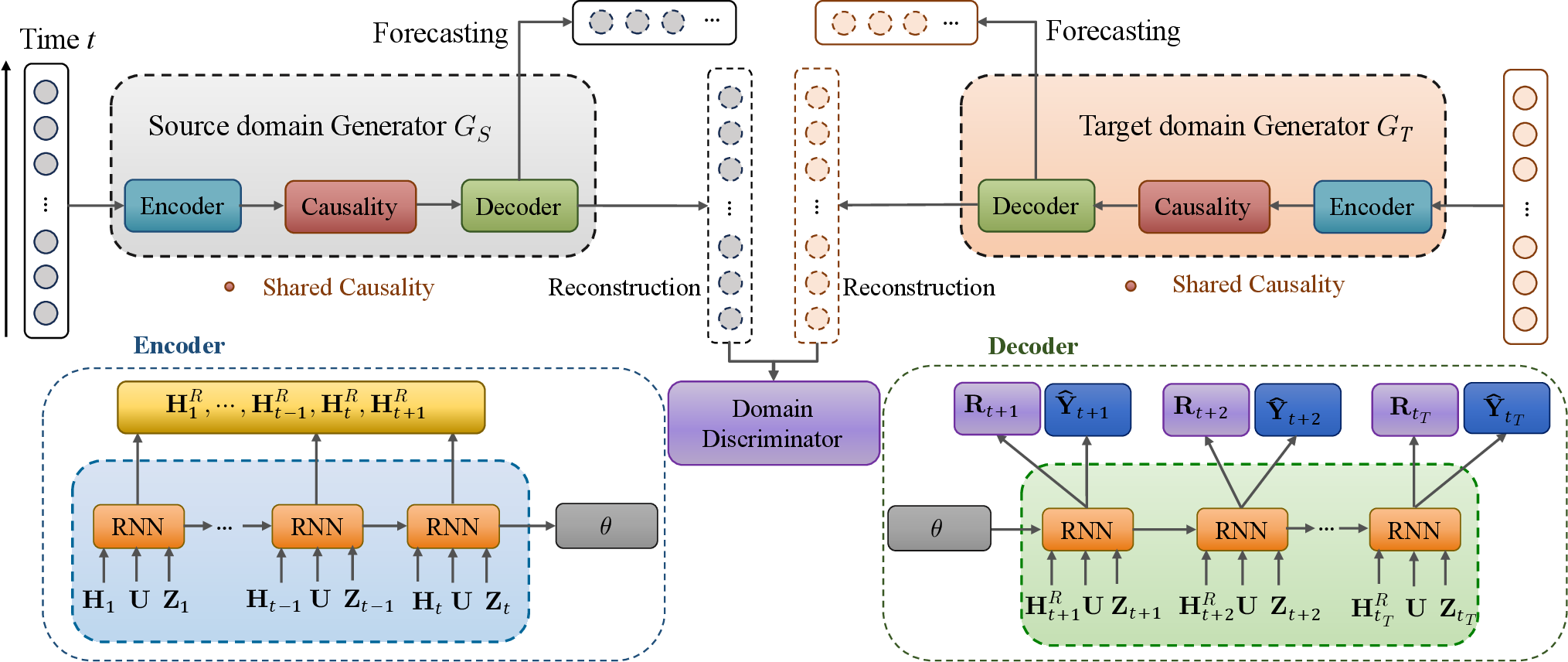}}
    \caption{An architecture overview of CDA forecaster}
    \label{figs:cda}     
\end{figure*}
\subsection{Joint Treatment-Outcome Model}
To estimate both statistical terms in Eq.(\ref{eqs:joint-model}), i.e., Treatment Term and Outcome Term, 
from observational time-series, we propose a joint treatment-outcome model, 
combining a counterfactual inference process and conditional domain adaptation process. 

\noindent \textbf{Shared Causality Module in Treatment Term and Outcome Term.}  
We design a causality module to be shared by both source domain and target domain, 
since treatment term and outcome term have different distributions but share same causal structure,
where the data in treatment term and outcome term with subscript $\leq t$ corresponds to source domain, and 
the data with subscript $>t$ corresponds to target domain.

For \textbf{Treatment Term} in Eq.(\ref{eqs:joint-model}), 
it consists of two sub-term, i.e., $P(\mathbf{X}_{>t}|\mathbf{H}_{\leq t},\mathbf{Z}_{>t})$
and $P(\mathbf{Z}_{>t}|\mathbf{H}_{\leq t})$ called treatment-covariate term and treatment-policy term.
For treatment-policy term, treatment policy can be artificially specified in advance, 
thus it is simplified to $\mathbf{Z}_{>t}=\{\mathbf{z}_{t}\}_{t+1}^{t+\tau}$.
For treatment-covariate term in domain adaptation, its different distributions denoted by 
$P(\mathbf{X}_{>t}^T|\mathbf{H}_{\leq t}^S,\mathbf{Z}_{>t}^T)\neq P(\mathbf{X}_{\leq t}^S)$
since the treatment policy $\mathbf{Z}_{<t}^S$ $(\subseteq\mathbf{H}_{\leq t}^S)$ in source domain
may be distinct from the treatment policy $\mathbf{Z}_{>t}^T$ in target domain, e.g., human intervention.

For \textbf{Outcome Term} in Eq.(\ref{eqs:joint-model}) under domain adaptation, 
its different distributions denoted by 
$P(\mathbf{Y}_{>t}^T|\mathbf{H}_{\leq t}^S,\mathbf{X}_{>t}^T,\mathbf{Z}_{>t}^T)\neq P(\mathbf{Y}_{\leq t}^S)$
since $\mathbf{X}_{\leq t}^S$ in source domain is distinct from $\mathbf{X}_{>t}^T$ in target domain
caused by the difference between treatment policies $\mathbf{Z}$.

In industrial time-series forecasting, although there is domain shift existing in treatment term and 
outcome term, domain-invariant representation can be extracted from causal mechanism perspective 
between domains. Shown in Fig.\ref{figs:causality}, the domain-invariant causality mean that the 
intrinsic causality will not change along with time $t$ or domain $S, T$, including 
$\mathbf{H}_{\leq t},\mathbf{Z}_{>t}$ are the direct cause of $\mathbf{X}_{>t}$,
while $\mathbf{X}_{>t}$ is the direct cause of $\mathbf{Y}_{>t}$.
Under this view, these causality indicates that treatment policy $\mathbf{Z}_{>t}^T$ 
is a primary cause for domain drift exlcuding the history-matching $\mathbf{H}_{\leq t}$,. 
Since we cast the industrial time-series forecasting in terms of a domain adaptation problem, 
the query of attention mechanism is accessible refer to the treatment policy $\mathbf{Z}_{>t}$. 
For simplification, the time-varying treatment $\mathbf{Z}_{>t}$ can be seen as 
the answers of queries $\mathbf{A}_{>t}=\{\mathbf{a}_{t}\}_{>t}$, 
recorded as the answer-based attention mechanism. 

\begin{definition}[\textbf{Position-wise CATE}]
    Let $\mathbf{z}_{t}$ be the assigned treatments over in timestep $t$, then the position-wise CATE 
    is defined as average causal effect of $\mathbf{X}_{t+1}$ under conditions 
    $\mathbf{H}_{\leq t}$ and do-operator $\mathrm{do}(\mathbf{z}_{t})$,
    \begin{equation}
            \mathrm{CATE}(\mathbf{H}_{\leq t},\mathbf{z}_{t})=\mathbb{E}[\mathbf{X}_{t+1}|\mathbf{H}_{\leq t},
            \mathrm{do}(\mathbf{z}_{t})]
            -\mathbb{E}[\mathbf{X}_{t+1}|\mathbf{H}_{\leq t},\mathbf{z}_{t}]
        \label{eqs:CATE}
    \end{equation}
    \label{def:CATE}
\end{definition}

\noindent \textbf{Reconstruction of Treatment Term and Outcome Term.} 
Considering the causal effect of $\mathbf{z}_{t}$ impacting on $\mathbf{X}_{t+1}$, a position-wise CATE 
(definition \ref{def:CATE}) is defined by utilizing $\mathbf{H}_{\leq t}$ and can be used to 
provide the keys $\mathbf{k}_{t+1}$ for attention mechanism,
\begin{equation}
    (\mathbf{a}_{t},\mathbf{k}_{t})=\left(\mathbf{z}_{t},\mathrm{CATE}(\mathbf{H}_{\leq t},\mathbf{z}_{t})\right)
    \label{eqs:answer-based}
\end{equation}

Thus a causal score $\alpha$, is computed as the normalized alignment between the answer 
$\mathbf{a}_{t}$ and keys $\mathbf{k}_{t\prime}$ at neighborhood positions $t^{\prime}\in\mathcal{N}(t)$ 
using an exponential position-wise CATE as kernel $\mathcal{K}(\cdot,\cdot)$.
\begin{equation}
    \alpha(\mathbf{a}_{t},\mathbf{k}_{t\prime})=\frac{\mathcal{K}(\mathbf{a}_{t},\mathbf{k}_{t})}
    {\sum_{t\prime\in\mathcal{N}(t)}\mathcal{K}(\mathbf{a}_{t},\mathbf{k}_{t\prime})}
    \label{eqs:causal-score}
\end{equation}
where ${\mathcal{K}}(\mathbf{a}_{t},\mathbf{k}_{t})=\exp\Bigl(\mathbf{k}_{t},(\mathbf{H}_{\leq t},\mathbf{a}_{t})\Bigr)$
and ${\mathcal{K}}(\mathbf{a}_{t},\mathbf{k}_{t})=\exp\Bigl(\mathbf{k}_{t}(\mathbf{H}_{\leq t},\mathbf{a}_{t})\Bigr)$.

For \textbf{Treatment Term}, based on the attention mechanism, 
the reconstruction $\mathbf{R}_{t}$ of $\mathbf{X}_{t}$ is produced as the 
value $\mathbf{X}_{t}$ weighted by causal score $\alpha(\mathbf{a}_{t},\mathbf{k}_{t\prime})$, followed by
\begin{equation}
    \mathbf{R}_{t|\mathbf{a}_{t}}=\alpha(\mathbf{a}_{t},\mathbf{k}_{t^{\prime}})\mathbf{X}_\mathcal{T},t=1,2,\cdots
    \label{eqs:reconstruction-x}
\end{equation}
Shown in Eq.(\ref{eqs:reconstruction-x}), the reconstruction $\mathbf{R}_{t}$ depend on the answer $\mathbf{a}_{t}$, 
namely, treatment policy $\mathbf{z}_{t}$. Thus, by assigned the available weights for rescontructing the 
covariates $\mathbf{X}_{t}$, the domain discrepancy and time-lags caused by treatment policy $\mathbf{z}_{t}$
can be negligible in industrial time-series forecasting.

\begin{definition}[\textbf{Conditional Markov Model for Sequence Outcomes}]
    The treatment-outcome assigned with treatment policy follows a Conditional Markov Model,
    \begin{equation}
        P(\mathbf{Y}_{>t}|\mathbf{H}_{\leq t},\mathbf{X}_{>t},\mathbf{Z}_{>t})=\prod_t
        P(\mathbf{Y}_{t+1}|\mathbf{H}_{\leq t},\mathbf{X}_{t+1},\mathbf{Z}_{t})
        \label{eqs:CMM}
    \end{equation}
    \label{def:CMM}
\end{definition}

For \textbf{Outcome Term}, Conditional Markov Model in definition \ref{def:CMM} is applied to achieve the
reconstruction of the outcome $\mathbf{Y}_{t+1}$, which is conditional on the observed covariates $\mathbf{H}_{\leq t}$, 
the latent covariates $\mathbf{X}_{t+1}$ and possible treatment $\mathbf{z}_{t}$. 
Once the reconstruction of treatment term, i.e., $\mathbf{R}_{t}$, has been obtained, 
the reconstruction of outcome term can be generalized via a sequence model 
\cite{hatt2021sequential} with CMM, shown as
\begin{equation}
    \mathbb{E}[\mathbf{Y}_{t+1}|\mathbf{H}_{\leq t}^\mathbf{R},\mathbf{R}_{t+1},
    \mathbf{Z}_{t}],t=1,2,\cdots 
    \label{eqs:reconstruction-y}
\end{equation}
where $\mathbf{H}_{\leq t}^{\mathbf{R}}=(\mathbf{R}_{1:t},\mathbf{Z}_{0:t-1},\mathbf{Y}_{1:t})$ denote 
the reconstruction of $\mathbf{H}_{\leq t}$. 

Once we obtained the reconstruction of treatment term and outcome term, the Eq.(\ref{eqs:reconstruction-y})
is applied to perform domain adaptation via an time-series model, called causal domain adaptation,
which aims at time-series forecasting based on answer-based attention mechanism 
discussed in following subsection.

\subsection{The Causal Domain Adaptation}
The proposed method Causal Domain Adaptation (CDA) Forecaster 
employs a sequence generator to make causal forecasting for industrial time-sereis. 
The domain-invariant causality shared by both domains guarantee the effectiveness in 
reconstructing treatment term and outcome term, 
while the reconstructions of both terms promote the performance of the learned 
domain-variant representation in time-series forecasting across domains. 
Fig.\ref{figs:cda} illustrates an overview of the proposed architecture.

\noindent \textbf{Adversarial Domain Adaptation.}
To compute the desired target prediction $\mathbf{Y}_{t}$ with limited time-series data, 
the manners of adversarial training on both domains are employed to 
formulize the following minimax problem, 
\begin{equation}
    \begin{split}
        \operatorname*{min}_{\mathcal{G}_\mathcal{S},\mathcal{G}_\mathcal{T}}\operatorname*{max}_{B} & \ \mathcal{L}_{seq}(\mathcal{D}_\mathcal{S};\mathcal{G}_\mathcal{S})+
        \mathcal{L}_{seq}(\mathcal{D}_\mathcal{T};\mathcal{G}_\mathcal{T})- \\ & \lambda\mathcal{L}_{dom}(\mathcal{D}_\mathcal{S},\mathcal{D}_\mathcal{T};B,\mathcal{G}_\mathcal{S},\mathcal{G}_\mathcal{T})       
    \end{split}
    \label{eqs:loss}
\end{equation}
where $\mathcal{G}_\mathcal{S}$ and $\mathcal{G}_\mathcal{T}$ denote the sequence generators that forecast the sequence 
for source domain dataset $\mathcal{D}_\mathcal{S}$ and target domain dataset $\mathcal{D}_\mathcal{T}$, respectively; 
parameter $\lambda\geq 0$ balances the estimation error $\mathcal{L}_{seq}$ and 
domain classification error $\mathcal{L}_{dom}$, and $B$ denotes a discriminator that classfies 
the domain between source and target.

Firstly, the loss term $\mathcal{L}_{seq}(\cdot)$ is induced by a sequence generator G as follows,
\begin{equation}
    \mathcal{L}_{seq}(\mathcal{D};\mathcal{G})=\sum_{i=1}^{N}\left(\frac{1}{T}
    \sum_{t=1}^{T}l(y_{i,t}^{\mathcal{D}},y_{i,t}^\mathcal{G})
    +\frac{1}{\tau}\sum_{t=\tau+1}^{T+\tau}l(y_{i,t}^{\mathcal{D}},y_{i,t}^\mathcal{G})\right)
    \label{eqs:loss_seq}
\end{equation}
where $l(\cdot)$ is a loss function and $y_{i,t}^{\mathcal{D}}, y_{i,t}^\mathcal{\mathcal{G}}$ 
are the actual sequence and estimated sequence provided by generator $\mathcal{G}$, respectively.

\begin{definition}[\textbf{Conditional Maximum Mean Discrepancy}]
    Given source $\mathcal{S}$ and target domain $\mathcal{T}$ 
    on the policy space $\mathbf{X}$ and policy condition $\mathbf{Z}$, 
    leading to the conditional maximum mean discrepancy with reproduction Kernel Hilbert Space $\mathcal{H}_{k}$,
    \begin{equation}
        d_{\mathrm{CMMD}}(\mathcal{S}|\mathbf{Z},\mathcal{T}|\mathbf{Z})=\|\mathbb{E}_{\mathbf{x}\sim\mathcal{S}}
        [\phi(\mathbf{X}|\mathbf{Z})]-\mathbb{E}_{\mathbf{X}\sim\mathcal{T}}[\phi(\mathbf{X}|\mathbf{Z})]
        \|_{\mathcal{H}_{k}}
        \label{eqs:cmmd}
    \end{equation}
    \label{def:cmmd}
    where $\phi(\mathbf{X}|\mathbf{Z}) \in \mathcal{H}$ is feature reconstruction map 
    associated with same condition $\mathbf{Z}$ for source and target domain.
\end{definition}

\begin{theorem}
    \label{the:the1}
    Let $\mu[P]$ be a distribution of $P$ in RKHS $\mathcal{H}_{k}$, 
    then via the reproducing property of RKHS $\mathcal{H}_{k}$, 
    we have 
    $\langle\phi, \mu[P_{\mathcal{T},\mathcal{S}}]\rangle=\mathbb{E}_{\mathbf{X}\sim\mathcal{T},\mathcal{S}}[\phi(\mathbf{X})]$
    for MMD,
    $\langle\phi, \mu[P_{\mathcal{T},\mathcal{S}|\mathbf{Z}}]\rangle=\mathbb{E}_{\mathbf{X}\sim\mathcal{T},\mathcal{S}}[\phi(\mathbf{X}|\mathbf{Z})]$
    for CMMD.
    Thus, the empirical estimate of squared CMMD shown in Definition \ref{def:cmmd} can be further simplifed as
    \begin{equation}
        \label{eqs:the1}
        \begin{aligned} 
            & \ \ \ \ d_{\mathrm{CMMD}}^2(\mathcal{S}|\mathbf{Z}, \mathcal{T}|\mathbf{Z}) \\
            & \leq \frac{1}{4}[d_{\mathrm{CMMD}}^2(\mathcal{S}|\mathbf{Z}, \mathcal{T}|\mathbf{Z})+2d_{\mathrm{CMMD}}(\mathcal{T}, \mathcal{T}|\mathbf{Z})d_{\mathrm{MMD}}(\mathcal{S}, \mathcal{T}) \\
            & \ \ \ + d_{\mathrm{CMMD}}^2(\mathcal{S}, \mathcal{S}| \mathbf{Z}) + d_{\mathrm{CMMD}}^2(\mathcal{T}, \mathcal{T}| \mathbf{Z}) + d_{\mathrm{MMD}}^2(\mathcal{S}, \mathcal{T})]
        \end{aligned} \nonumber
    \end{equation}
    where $d_{\mathrm{MMD}}$ is the standard Maximum Mean Distance (MMD),
    $d_{\mathrm{CMMD}}(\mathcal{S}, \mathcal{S}|\mathbf{Z})$ is applied to measure the closeness of source domain
    between the raw distributions $P_{\mathcal{S}}$ and condition distribution $P_{\mathcal{S}|\mathbf{Z}}$, and 
    $d_{\mathrm{CMMD}}(\mathcal{T}, \mathcal{T}|\mathbf{Z})$ is applied to measure the closeness of target domain
    between the raw distributions $P_{\mathcal{T}}$ and condition distribution $P_{\mathcal{T}|\mathbf{Z}}$.
\end{theorem}
\noindent The proof of Theorem \ref{the:the1} has been presented in Appendix.

\begin{remark}
    The property of CMMD shown in Theorem \ref{the:the1} ensure that CMMD estimate the 4 types of closeness, including
    conditional closeness between source and target domain, i.e., $d_{\mathrm{CMMD}}(\mathcal{S}|\mathbf{Z}, \mathcal{T}|\mathbf{Z})$,
    partial conditional closeness between source and target domain, i.e., $d_{\mathrm{CMMD}}(\mathcal{S}, \mathcal{S}|\mathbf{Z}), d_{\mathrm{CMMD}}(\mathcal{T}, \mathcal{T}|\mathbf{Z})$
    and non-condition loseness between source and target domain, i.e., $d_{\mathrm{MMD}}(\mathcal{S}, \mathcal{T})$.
\end{remark}

\begin{corollary}
    \label{cor:cor1}
    In domain adversarial learning, the MMD can be transformed into a more tractable form 
    from the perspective of loss function \cite{jin2022domain}, that is,
    \begin{equation}
        \begin{aligned}
            & \ \ \ \ d_{\mathrm{CMMD}}^2(\mathcal{S}|\mathbf{Z}, \mathcal{T}|\mathbf{Z}) \\
            & \Rightarrow \mathcal{L}(\mathcal{D}_{\mathcal{S}}, \mathcal{D}_{\mathcal{T}},\mathcal{D}_{\mathcal{S}|\mathbf{Z}}, \mathcal{D}_{\mathcal{T}|\mathbf{Z}};\mathcal{G}_{\mathcal{S}}, \mathcal{G}_{\mathcal{T}}) \\
            & 
            = \mathcal{L}_{1}(\mathcal{D}_{\mathcal{S}}, \mathcal{D}_{\mathcal{T}};\mathcal{G}_{\mathcal{S}}, \mathcal{G}_{\mathcal{T}})
            + \mathcal{L}_{2}(\mathcal{D}_{\mathcal{S}|\mathbf{Z}}, \mathcal{D}_{\mathcal{T}|\mathbf{Z}};\mathcal{G}_{\mathcal{S}}, \mathcal{G}_{\mathcal{T}}) \\
            & 
            + \mathcal{L}_{3}(\mathcal{D}_{\mathcal{S}},\mathcal{D}_{\mathcal{S}|\mathbf{Z}};\mathcal{G}_{\mathcal{S}})
            + \mathcal{L}_{4}(\mathcal{D}_{\mathcal{T}},\mathcal{D}_{\mathcal{T}|\mathbf{Z}};\mathcal{G}_{\mathcal{T}}) \\
            & 
            + \mathcal{L}_{1}(\mathcal{D}_{\mathcal{S}}, \mathcal{D}_{\mathcal{T}};\mathcal{G}_{\mathcal{S}}, \mathcal{G}_{\mathcal{T}})^{\frac{1}{2}}
            \mathcal{L}_{2}(\mathcal{D}_{\mathcal{S}|\mathbf{Z}}, \mathcal{D}_{\mathcal{T}|\mathbf{Z}};\mathcal{G}_{\mathcal{S}}, \mathcal{G}_{\mathcal{T}})^{\frac{1}{2}}
        \end{aligned}
    \end{equation}
    where $\mathcal{D}_{\mathcal{S}}=\{\mathcal{S}, \mathcal{\hat{S}}\}, \mathcal{D}_{\mathcal{T}}=\{\mathcal{T}, \mathcal{\hat{T}}\}$, 
    $\mathcal{D}_{\mathcal{S}|\mathbf{Z}}=\{\mathcal{S}, \mathcal{\hat{S}}\}_{|\mathbf{Z}}$
    and $\mathcal{D}_{\mathcal{T}|\mathbf{Z}}=\{\mathcal{T}, \mathcal{\hat{T}}\}_{|\mathbf{Z}}$,
    which $\mathcal{\hat{S}}$ and $\mathcal{\hat{T}}$ 
    denote the generated source and target domains by sequence generators of source domain $\mathcal{G}_{\mathcal{S}}$
    and target domain $\mathcal{G}_{\mathcal{T}}$.
    \begin{equation}
        \left\{
        \begin{aligned}
        \mathcal{L}_{1}
        & =\beta_{1}\left\|\frac{1}{|\mathcal{D}_{\mathcal{S}}|}\sum_{\mathbf{X}\in{\mathcal{D}_{\mathcal{S}}}} \mathbf{X}
        -\frac{1}{|\mathcal{D}_{\mathcal{T}}|}\sum_{\mathbf{X}\in\mathcal{D}_{\mathcal{T}}} \mathbf{X}\right\|_{2}^{2} \\
        \mathcal{L}_{2}
        & =\beta_{2}\left\|\frac{1}{|\mathcal{D}_{\mathcal{S}}|}\sum_{\mathbf{a}\in\mathbf{Z}}\sum_{\mathbf{X}\in{\mathcal{D}_{\mathcal{S}|\mathbf{Z}}}} \mathbf{R}_{\mathbf{a}}
        -\frac{1}{|\mathcal{D}_{\mathcal{T}}|}\sum_{\mathbf{a}\in\mathbf{Z}}\sum_{\mathbf{X}\in\mathcal{D}_{\mathcal{T}|\mathbf{Z}}} \mathbf{R}_{\mathbf{a}}\right\|_{2}^{2}\\
        \mathcal{L}_{3}
        &=\beta_{3}\left\|\frac{1}{|\mathcal{D}_{\mathcal{S}}|}\sum_{\mathbf{X}\in{\mathcal{D}_{\mathcal{S}}}} \mathbf{X}
        -\frac{1}{|\mathcal{D}_{\mathcal{S}|\mathbf{Z}}|}\sum_{\mathbf{a}\in\mathbf{Z}}\sum_{\mathbf{X}\in\mathcal{D}_{\mathcal{S}|\mathbf{Z}}} \mathbf{R}_{\mathbf{a}}\right\|_{2}^{2}\\
        \mathcal{L}_{4}
        &=\beta_{4}\left\|\frac{1}{|\mathcal{D}_{\mathcal{T}}|}\sum_{\mathbf{X}\in{\mathcal{D}_{\mathcal{T}}}} \mathbf{X}
        -\frac{1}{|\mathcal{D}_{\mathcal{T}|\mathbf{Z}}|}\sum_{\mathbf{a}\in\mathbf{Z}}\sum_{\mathbf{X}\in\mathcal{D}_{\mathcal{T}|\mathbf{Z}}} \mathbf{R}_{\mathbf{a}}\right\|_{2}^{2}      
        \end{aligned}
        \right.
    \end{equation}
    where $\mathbf{R}_{\mathbf{a}}:=\alpha(\mathbf{a}, \mathbf{k})\mathbf{X}$, 
    $\mathbf{k}$ denote the keys in transfer learning,
    $|\cdot|$ is the cardinality function.
    and the constants $\beta$ are the balance parameters.
\end{corollary}
\noindent The proof of Corollary \ref{cor:cor1} has been presented in Appendix.

\begin{remark}
    The unified loss function 
    $\mathcal{L}(\mathcal{D}_{\mathcal{S}}, \mathcal{D}_{\mathcal{T}},\mathcal{D}_{\mathcal{S}|\mathbf{Z}}, \mathcal{D}_{\mathcal{T}|\mathbf{Z}};\mathcal{G}_{\mathcal{S}}, \mathcal{G}_{\mathcal{T}})$
    in Corollary \ref{cor:cor1} consists of 4 types independent loss function, 
    $\mathcal{L}_{1}, \mathcal{L}_{2}, \mathcal{L}_{3}$ and $\mathcal{L}_{4}$.
    The $\mathcal{L}_{1}$ estimate the error between source and target domain,
    the $\mathcal{L}_{2}$ estimate the error between source and target domain with policy,
    the $\mathcal{L}_{3}$ estimate the error between source domain with policy and source domain without policy,
    the $\mathcal{L}_{4}$ estimate the error between target domain with policy and target domain without policy.
    Thus, $\mathcal{L}(\mathcal{D}_{\mathcal{S}}, \mathcal{D}_{\mathcal{T}},\mathcal{D}_{\mathcal{S}|\mathbf{Z}}, \mathcal{D}_{\mathcal{T}|\mathbf{Z}};\mathcal{G}_{\mathcal{S}}, \mathcal{G}_{\mathcal{T}})$ 
    can estimate the complete domain classification error.
\end{remark}

Subsequently, since the domain-invariant causality shared by both domains, 
and both terms has been reconstructed, 
a domain discriminator is introduced to recognize the decision boundary for source and target domains 
by given the reconstruction $\mathbf{R}_{t_{|\mathbf{a}_{t}}}$. 
Shown in Corollary \ref{cor:cor1}, a temporal loss function for domain classification error $\mathcal{L}_{dom}$ 
is defined to compute the distance between source and target domains under treatment policy $\mathbf{a}_{t}$,
\begin{equation}
    \label{eqs:loss_dis}
    \begin{aligned}
        & \ \ \ \ \mathcal{L}_{dom}(\mathcal{D}_{\mathcal{S}}, \mathcal{D}_{\mathcal{T}},\mathcal{D}_{\mathcal{S}|\mathbf{Z}}, \mathcal{D}_{\mathcal{T}|\mathbf{Z}};\mathcal{G}_{\mathcal{S}}, \mathcal{G}_{\mathcal{T}}) \\
        &=\left\|\frac{1}{|\mathcal{D}_{\mathcal{S}}|}\sum_{\mathbf{X}_{t}\in{\mathcal{D}_{\mathcal{S}}}} \mathbf{X}_{t}
        -\frac{1}{|\mathcal{D}_{\mathcal{T}}|}\sum_{\mathbf{X}_{t}\in\mathcal{D}_{\mathcal{T}}} \mathbf{X}_{t}\right\|_{2}^{2} \\
        &+\left\|\frac{1}{|\mathcal{D}_{\mathcal{S}}|}\sum_{\mathbf{a}_{t}\in\mathbf{Z}}\sum_{\mathbf{X}_{t}\in{\mathcal{D}_{\mathcal{S}|\mathbf{Z}}}} \mathbf{R}_{t|{\mathbf{a}_{t}}}
        -\frac{1}{|\mathcal{D}_{\mathcal{T}}|}\sum_{\mathbf{a}_{t}\in\mathbf{Z}}\sum_{\mathbf{X}_{t}\in\mathcal{D}_{\mathcal{T}|\mathbf{Z}}} \mathbf{R}_{t|{\mathbf{a}_{t}}}\right\|_{2}^{2}\\    
        &+\left\|\frac{1}{|\mathcal{D}_{\mathcal{S}}|}\sum_{\mathbf{X}_{t}\in{\mathcal{D}_{\mathcal{S}}}} \mathbf{X}_{t}
        -\frac{1}{|\mathcal{D}_{\mathcal{S}|\mathbf{Z}}|}\sum_{\mathbf{a}_{t}\in\mathbf{Z}}\sum_{\mathbf{X}_{t}\in\mathcal{D}_{\mathcal{S}|\mathbf{Z}}} \mathbf{R}_{t|{\mathbf{a}_{t}}}\right\|_{2}^{2}\\
        &+\left\|\frac{1}{|\mathcal{D}_{\mathcal{T}}|}\sum_{\mathbf{X}_{t}\in{\mathcal{D}_{\mathcal{T}}}} \mathbf{X}_{t}
        -\frac{1}{|\mathcal{D}_{\mathcal{T}|\mathbf{Z}}|}\sum_{\mathbf{a}_{t}\in\mathbf{Z}}\sum_{\mathbf{X}_{t}\in\mathcal{D}_{\mathcal{T}|\mathbf{Z}}} \mathbf{R}_{t|{\mathbf{a}_{t}}}\right\|_{2}^{2}      
    \end{aligned}
\end{equation}
where $|\mathcal{X}^{S}|$ and $|\mathcal{X}^{T}|$ denote the cardinality associated with 
the source $\mathcal{D}_\mathcal{S}$ and target $\mathcal{D}_\mathcal{T}$.

\noindent \textbf{Domain Adversarial Learning.} 
Recalling the problem formulation, we have defined a joint treatment-outcome model Eq.(\ref{eqs:joint-model}) 
based on the shared causality and Markov model. 
The shared causality induces the domain-invariant features $\mathbf{R}_{t_{|\mathbf{a}_{t}}}$ calculated by 
the answers of queries $\mathbf{A}|\mathbf{Z}$ and keys $\mathbf{K}|\mathrm{CATE}$ across domains. 
While a domain discriminator tries to classify the domain between source and target data, 
the defined sequence generators in this section are trained to confuse the discriminator, 
aiming at the better generalization to a worst-case set of discriminator. 
By adopting the MSE loss for $l(\cdot)$ , 
the minimax objective in Eq.(\ref{eqs:loss}) is now formally defined over generators $\mathcal{G}_\mathcal{S},\mathcal{G}_\mathcal{T}$
with parameter $\Theta_\mathcal{G}$ and domain discriminator $B$ with parameter $\Theta_{B}$. 
Algorithm \ref{alg:cda} summarized the training routine of CDA. 
The CDA alternately update the $\Theta_\mathcal{G}$ and $\Theta_{B}$ in opposite directions 
so that $\mathcal{G}_\mathcal{S},\mathcal{G}_\mathcal{T}$ and $B$ are trained adversarially. 
Here, $\widehat{\mathbf{X}},\widehat{\mathbf{Y}}$ denote the predictive value of $\mathbf{X}, \mathbf{Y}$
for sequence generators.

\begin{algorithm*}[!t] 
    \caption{\textbf{Adversarial Training of CDA}} 
    \label{alg:cda}
    \begin{algorithmic}[1]
    \Require datasets $\mathcal{D}_\mathcal{S},\mathcal{D}_\mathcal{T}$, epoches $E$, setp sizes
    \Ensure Trained CDA forecaster parameterized by $\Theta_\mathcal{G}$ and $\Theta_{B}$
    \State \textbf{Initialization:} parameters $\Theta_\mathcal{G}$ and $\Theta_{B}$
    \For {$\mathrm{epoch}=1$ to $E$} 
        \Repeat
            \State sample $(\mathbf{X}_\mathcal{S},\mathbf{Y}_\mathcal{S},\mathbf{Z}_\mathcal{S})\sim \mathcal{D}_\mathcal{S}$ 
                    and $(\mathbf{X}_\mathcal{T},\mathbf{Y}_\mathcal{T},\mathbf{Z}_\mathcal{T})\sim \mathcal{D}_\mathcal{T}$
            \State generate $\mathbf{R}_\mathcal{S},\widehat{\mathbf{Y}}_\mathcal{S}=\mathcal{G}_\mathcal{S}(\mathbf{X}_\mathcal{S},\mathbf{Z}_\mathcal{S})$ 
                    and $\mathbf{R}_\mathcal{T},\widehat{\mathbf{Y}}_\mathcal{T}=\mathcal{G}_\mathcal{T}(\mathbf{X}_\mathcal{T},\mathbf{Z}_\mathcal{T})$
            \State compute $\mathcal{L}_{seq}(\mathbf{Y}_\mathcal{S}, \widehat{\mathbf{Y}}_\mathcal{S})$, 
                    $\mathcal{L}_{seq}(\mathbf{Y}_\mathcal{T}, \widehat{\mathbf{Y}}_\mathcal{T})$ for $\mathcal{G}_\mathcal{S}, \mathcal{G}_\mathcal{T}$,        
                    $\mathcal{L}_{dom}(\mathbf{X}_\mathcal{S}, \mathbf{X}_\mathcal{T}, \mathbf{Z}_\mathcal{S}, \mathbf{Z}_\mathcal{T})$ for $B$ 
                    \Comment{Using Eq.(\ref{eqs:loss_dis}) and (\ref{eqs:loss_seq})}
            \State compute $\mathcal{L}=\mathcal{L}_{seq}(\mathbf{Y}_\mathcal{S}, \widehat{\mathbf{Y}}_\mathcal{S})
                    +\mathcal{L}_{seq}(\mathbf{Y}_\mathcal{T}, \widehat{\mathbf{Y}}_\mathcal{T})
                    +\lambda\mathcal{L}_{dom}(\mathbf{R}_\mathcal{S}, \mathbf{R}_\mathcal{T}, \mathbf{Z}_\mathcal{S}, \mathbf{Z}_\mathcal{T})$ 
                    \Comment{Using Eq.(\ref{eqs:loss})}
            \State gradient descent with $\nabla_{\Theta_\mathcal{G}}\mathcal{L}$ to update $\mathcal{G}_\mathcal{S}, \mathcal{G}_\mathcal{T}$
            \State gradient descent with $\nabla_{\Theta_{B}}\mathcal{L}$ to update $B$
        \Until $\mathcal{D}_\mathcal{T}$ is exhausted
    \EndFor
  \end{algorithmic}
\end{algorithm*}

\section{Experimental Verification}
\label{Experiments}
In this section, we verify the effectiveness of the proposed CDA in adapting from source domain to 
target domain, and then estimate counterfactual outcome trajectories under treatment policy. 
The experiments proceeding in real-word oilfield dataset are composed of two parts, 
that is, 
time-series forecasting and the optimal policy exploiting for oil production, respectively.

\begin{figure}[!t]
    \centering
    \includegraphics[width=0.5\textwidth, height=5cm]{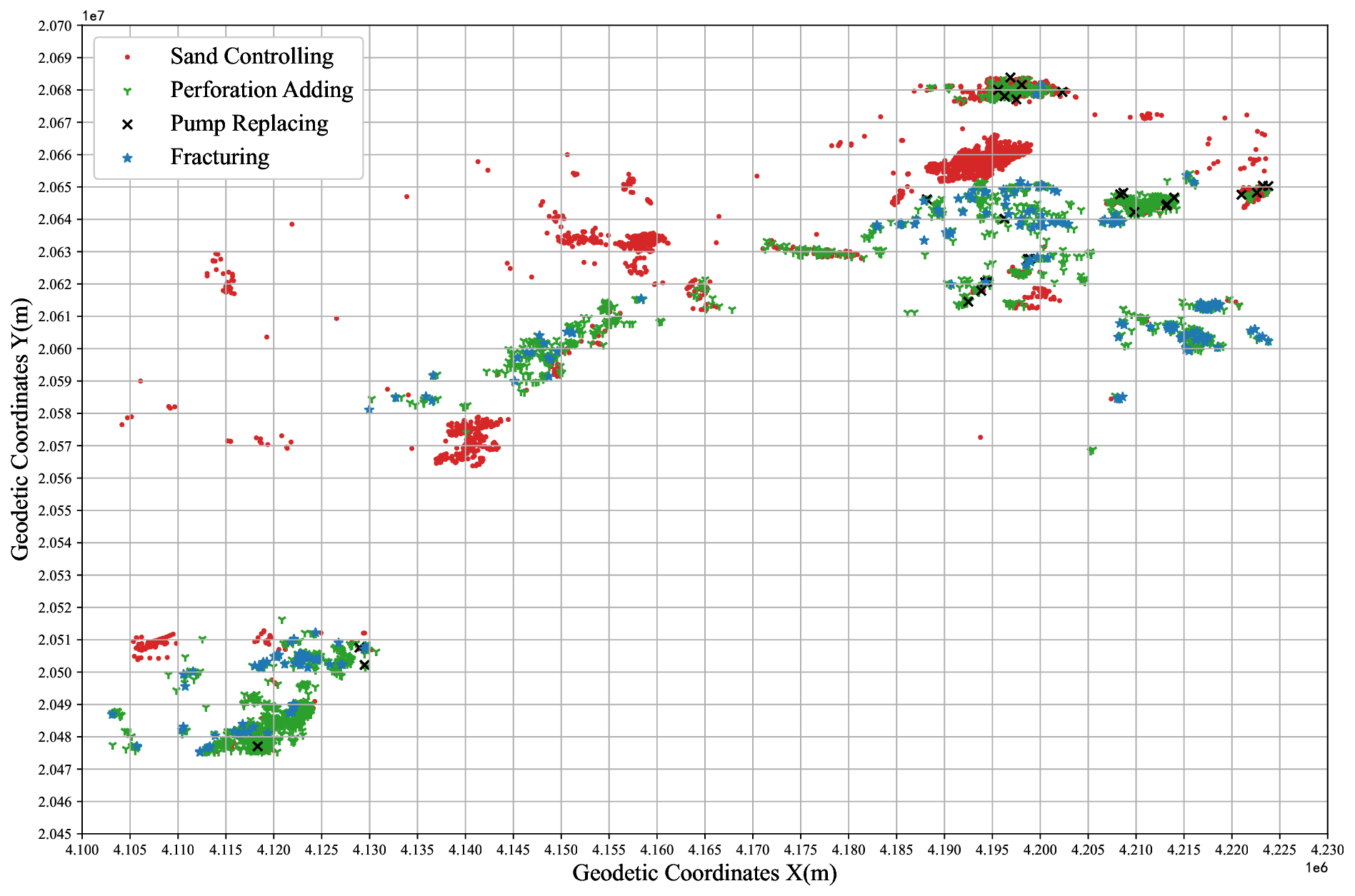}
    \caption{The location of oil wells.}
    \label{figs:platform}     
\end{figure}

\subsection{Experimental Apparatus}
\noindent \textbf{Datasets in oil and gas industry.} 
The target oil block , GD-Ng52+3, is a typical sample of waterflooding oilfields in China . 
It is located  in East China, which has well developed sand body and thick oil layer. 
The area of this block is about 9.1 $\mathrm{km}^{2}$, with an average effective thickness of 7.9 $\mathrm{m}$, 
holding the geological reserves of 13150 million tons. 
This block was developed by water flooding since 1986, 
which has a long history and adequate development data. 
Now it is at the stage of high water cut, 
which faces a critical issue about how to optimze the water inject/oil product scheme.
The real-world oilfield datasets are sampled from the above oil block in Shandong Province, China. 
Notably, the oilfield dataset are synthetic 
exploitation and production monthly time-series of 1474 wells from 2001 to 2020, 
comprising a total of 353760 recods. 
The dataset contains 14 indicators, covering dynamical engineering variable $\mathbf{X}$,
dynamical policy variable $\mathbf{Z}$, dynamical production variable $\mathbf{Y}$ 
and static geological variable $\mathbf{U}$, shown in Table \ref{tab:datasets}. 
In oil and gas industry, the policy can be dividided into the 4 typical classification,
i.e., \textit{Sand Controlling, Perforation Adding, Pump Replacing} and \textit{Fracturing}, 
where the top 3 policies are the productive measure and the last policy is the protective policy.
Thus, the dataset can be dividide to 4 partions according to the property of policy $\mathbf{Z}$.
The characteristics of variables w.r.t. 4 partions are described in Table \ref{tab:statistic}

The ubiquity of temporal causality exists in oil and gas industry, but it is accompanied by the time lags 
and indirect causal effect due to the properties of oilfield exploitation. 
Specifically, oil production $\mathbf{Y}$ has been demonstrated as lags in response when one of treatment 
policy $\mathbf{Z}$ is applied in activity, and policy indirectly affects oil production $\mathbf{Y}$ by 
directly influencing engineering factors $\mathbf{X}$.

\noindent \textbf{Experimental settings} 
The experiments consist of 3 independent phase for oil production $\mathbf{Y}$, including
inside-well time-sereis forecasting, cross-well time-sereis predicting 
and optimal policy determining.
\begin{itemize}
    \item The experiment, inside-well time-sereis forecasting, 
    is to forecast the oil production in the coming months
    by utilizing the history time-series of whole oil wells. 
    \item The experiment, cross-well time-sereis predicting, 
    is to predict the oil production in the entire period 
    by utilizing the history time-series of partial oil wells.
    \item The experiment, optimal policy determining, 
    is to determine the optimal policy in improving oil production for a specifical well.
\end{itemize}
For inside-well time-sereis forecasting and cross-well time-sereis predicting,
we compare CDA with the single-domain and cross-domain forecaster. 
The conventional single-domain forecasters trained only on source or target domain, 
including DecoderMLP, LSTM, GRU, Bi-LSTM, Bi-GRU, DeepAR and N-Beats. 
The cross-domain forecasters trained on both source and target domain, 
including Transformer, TFT, DANN, MMDA, SASA and our CDA. 
Specifically, for each experiment, 
the source domain $\mathcal{D}_\mathcal{S}$ is $\left[\mathbf{X}, \mathbf{Z}, \mathbf{U}\right]$
for the models TFT and ours, 
while the source domain of others $\mathcal{D}_\mathcal{S}$ is $\left[\mathbf{X}, \mathbf{Z}\right]$.
In addition, $R^{2}$ score, Root Mean Square Error (RMSE) and Mean Absolute Error (MAE) are employed as 
evaluation metrics to assess the performance of models, which are as follows,
\begin{equation}
    \begin{aligned}
        \mathrm{R^{2}}&=1-\mathrm{SSE/SST} \\
        \mathrm{RMSE}&=\left(\sum\nolimits_{t=1}^{n}(\hat{y}_{t}-y_{t})/n\right)^{1/2} \\
        \mathrm{MAE}&=\frac{1}{n}\sum\nolimits_{t=1}^{n}|\hat{y}_{t}-y_{t}| \\
    \end{aligned}
    \label{eqs:metrics}
\end{equation}
where $\mathrm{SSE}=\sum\nolimits_{t=1}^n(\hat{y}_\mathcal{T}-y_\mathcal{T})^2,
\mathrm{SST}=\sum\nolimits_{t=1}^n(\bar{y}_\mathcal{T}-y_\mathcal{T})^2$; 
$y_{t},{\widehat{y}}_{t}$, and ${\bar{y}}_{t}$ denote the actual value, predictive value, 
and mean value at timestep $t$.

\begin{table*}[!t]
    \caption{Categorization and variables of datasets}
    \centering
    \label{tab:datasets}
    \begin{threeparttable}
    \begin{tabular}{c|c|c}
    \bottomrule
    \textbf{Category}    & \textbf{Variables \tnote{1}}        & \textbf{The Remarks of $\mathbf{Z}$ \tnote{2}} \\ \hline 
    \multirow{3}{*}{Dynamical} & \makecell[l]{
    $\mathbf{X}_{1}$, Days of produced $(\mathrm{d})$,\\ 
    $\mathbf{X}_{2}$, Effective thickness $(\mathrm{m})$,\\
    $\mathbf{X}_{3}$, Pump depth $(\mathrm{m})$,\\ 
    $\mathbf{X}_{4}$, Pump diameter $(\mathrm{m})$,\\ 
    $\mathbf{X}_{5}$, Pump efficiency $(\mathrm{\%})$,\\
    $\mathbf{X}_{6}$, Displacement $(\mathrm{t})$ \\ 
    $\mathbf{X}_{7}$, Dynamic fluid interface $(\mathrm{m})$ \\
    $\mathbf{X}_{8}$, Stroke $(\mathrm{m})$,\\ 
    $\mathbf{X}_{9}$, Frequency of stroke $(\mathrm{SPM})$,\\
    $\mathbf{X}_{10}$, Casing pressure $(\mathrm{MPa})$
    } & \multirow{4}{*}{\makecell[l]{\textbf{Sand Controlling}, i.e., 
    Protecting the relevant instruments of oil wells, \\
    and slightly improving the capacity of producing ; \\
    \textbf{Perforation Adding}, i.e., Enlarging the contact area of reservoir,  \\ 
    and improving the skin factor; \\
    \textbf{Pump Replacing}, i.e., Changing the pressure difference in producing, \\
    and improving the capacity of supplying liquid; \\
    \textbf{Fracturing}, i.e., Improving the permeability, and enhancing the capacity \\ 
    of liquid flowing in the formation. \\
    }} \\ \cline{2-2}
                       & \makecell[l]{
                       $\mathbf{Z}_{1}$, Sand Controlling, \\ 
                       $\mathbf{Z}_{2}$, Perforation Adding, \\ 
                       $\mathbf{Z}_{3}$, Pump Replacing, \\ 
                       $\mathbf{Z}_{4}$, Fracturing} &                    \\ \cline{2-2}
                       & \makecell[l]{$\mathbf{Y}$, Monthly Oil Production $(\mathrm{t})$} &                 \\ \cline{1-2}
    {Static}                  & \makecell[l]{
        $\mathbf{U}_{1}$,  Formation temperature (\textcelsius),\\
        $\mathbf{U}_{2}$,  Formation pressure$(\mathrm{MPa})$} &                    \\ 
    \toprule
    \end{tabular}
    \begin{tablenotes}
    \footnotesize
    \item[1] {The vectors $\mathbf{X}=\{\mathbf{X}_{1}, \mathbf{X}_{2}, \cdots, \mathbf{X}_{10}\}$, 
    $\mathbf{Z}=\{\mathbf{Z}_{1}, \mathbf{Z}_{2}, \mathbf{Z}_{3}, \mathbf{Z}_{4}\}$, 
    $\mathbf{U}=\{\mathbf{U}_{1}, \mathbf{U}_{2}\}$.}
    \item[2] {(Can be classified into four types)
    The policies, including Perforation Adding, Pump Replacing and Fracturing, 
    aim at improving the production capacity, 
    while Sand Controlling is utilized to protect the instruments of oil wells.}
    \end{tablenotes}
    \end{threeparttable}
\end{table*}

\begin{table*}[!t]
    \caption{The characteristics of oilfield datasets for 4 policies}
    \centering
    \label{tab:statistic}
    \begin{tabular}{c|cccccc|cccccc}
    \bottomrule
    \multirow{2}{*}{Variables} & \multicolumn{6}{c|}{\textbf{Sand Controlling}}                                                                                                                                                       & \multicolumn{6}{c}{\textbf{Perforation Adding}}                                                                                                                                                     \\ \cline{2-13} 
                                     & \multicolumn{1}{c}{\textbf{min}} & \multicolumn{1}{c}{\textbf{max}} & \multicolumn{1}{c}{\textbf{mean}} & \multicolumn{1}{c}{\textbf{std}} & \multicolumn{1}{c}{\textbf{count}} & \multicolumn{1}{c|}{\textbf{num}} & \multicolumn{1}{c}{\textbf{min}} & \multicolumn{1}{c}{\textbf{max}} & \multicolumn{1}{c}{\textbf{mean}} & \multicolumn{1}{c}{\textbf{std}} & \multicolumn{1}{c}{\textbf{count}} & \multicolumn{1}{c}{\textbf{num}} \\ \hline
    $\mathbf{X}_{1}$                 & \multicolumn{1}{c}{0}            & \multicolumn{1}{c}{31}           & \multicolumn{1}{c}{20}            & \multicolumn{1}{c}{12}           & \multicolumn{1}{c}{234000}         & 975           & \multicolumn{1}{c}{0}            & \multicolumn{1}{c}{31}           & \multicolumn{1}{c}{12}            & \multicolumn{1}{c}{22}           & \multicolumn{1}{c}{101760}         & 424           \\ 
    $\mathbf{X}_{2}$                 & \multicolumn{1}{c}{0}            & \multicolumn{1}{c}{2017.6}       & \multicolumn{1}{c}{8.24}          & \multicolumn{1}{c}{15.12}        & \multicolumn{1}{c}{234000}         & 975           & \multicolumn{1}{c}{0}            & \multicolumn{1}{c}{455}          & \multicolumn{1}{c}{17.59}         & \multicolumn{1}{c}{12.1}         & \multicolumn{1}{c}{101760}         & 424           \\ 
    $\mathbf{X}_{3}$                 & \multicolumn{1}{c}{0}            & \multicolumn{1}{c}{14149.3}      & \multicolumn{1}{c}{950}           & \multicolumn{1}{c}{281.1}        & \multicolumn{1}{c}{234000}         & 975           & \multicolumn{1}{c}{0}            & \multicolumn{1}{c}{8960.6}       & \multicolumn{1}{c}{452.2}         & \multicolumn{1}{c}{1311}         & \multicolumn{1}{c}{101760}         & 424           \\ 
    $\mathbf{X}_{4}$                 & \multicolumn{1}{c}{0}            & \multicolumn{1}{c}{997}          & \multicolumn{1}{c}{62.05}         & \multicolumn{1}{c}{27.4}         & \multicolumn{1}{c}{234000}         & 975           & \multicolumn{1}{c}{0}            & \multicolumn{1}{c}{150}          & \multicolumn{1}{c}{13.26}         & \multicolumn{1}{c}{52.91}        & \multicolumn{1}{c}{101760}         & 424           \\ 
    $\mathbf{X}_{5}$                 & \multicolumn{1}{c}{0}            & \multicolumn{1}{c}{212970}       & \multicolumn{1}{c}{44.5}          & \multicolumn{1}{c}{452.4}        & \multicolumn{1}{c}{234000}         & 975           & \multicolumn{1}{c}{0}            & \multicolumn{1}{c}{48690}        & \multicolumn{1}{c}{355.1}         & \multicolumn{1}{c}{54.6}         & \multicolumn{1}{c}{101760}         & 424           \\ 
    $\mathbf{X}_{6}$                 & \multicolumn{1}{c}{0}            & \multicolumn{1}{c}{17854.6}      & \multicolumn{1}{c}{60.8}          & \multicolumn{1}{c}{124.4}        & \multicolumn{1}{c}{234000}         & 975           & \multicolumn{1}{c}{0}            & \multicolumn{1}{c}{4003.2}       & \multicolumn{1}{c}{45.28}         & \multicolumn{1}{c}{49.11}        & \multicolumn{1}{c}{101760}         & 424           \\ 
    $\mathbf{X}_{7}$                 & \multicolumn{1}{c}{1}            & \multicolumn{1}{c}{8444}         & \multicolumn{1}{c}{578.9}         & \multicolumn{1}{c}{365.8}        & \multicolumn{1}{c}{234000}         & 975           & \multicolumn{1}{c}{1}            & \multicolumn{1}{c}{8756}         & \multicolumn{1}{c}{544.5}         & \multicolumn{1}{c}{829.6}        & \multicolumn{1}{c}{101760}         & 424           \\ 
    $\mathbf{X}_{8}$                 & \multicolumn{1}{c}{0}            & \multicolumn{1}{c}{99}           & \multicolumn{1}{c}{3.46}          & \multicolumn{1}{c}{6.27}         & \multicolumn{1}{c}{234000}         & 975           & \multicolumn{1}{c}{0}            & \multicolumn{1}{c}{99}           & \multicolumn{1}{c}{2.01}          & \multicolumn{1}{c}{3.44}         & \multicolumn{1}{c}{101760}         & 424           \\ 
    $\mathbf{X}_{9}$                 & \multicolumn{1}{c}{0}            & \multicolumn{1}{c}{306}          & \multicolumn{1}{c}{10.5}          & \multicolumn{1}{c}{21.94}        & \multicolumn{1}{c}{234000}         & 975           & \multicolumn{1}{c}{0}            & \multicolumn{1}{c}{120}          & \multicolumn{1}{c}{8.19}          & \multicolumn{1}{c}{5.89}         & \multicolumn{1}{c}{101760}         & 424           \\ 
    $\mathbf{X}_{10}$                & \multicolumn{1}{c}{0}            & \multicolumn{1}{c}{240}          & \multicolumn{1}{c}{1.14}          & \multicolumn{1}{c}{3.64}         & \multicolumn{1}{c}{234000}         & 975           & \multicolumn{1}{c}{0}            & \multicolumn{1}{c}{83}           & \multicolumn{1}{c}{2.91}          & \multicolumn{1}{c}{1.30}         & \multicolumn{1}{c}{101760}         & 424           \\ 
    $\mathbf{Y}$                     & \multicolumn{1}{c}{0}            & \multicolumn{1}{c}{8221}         & \multicolumn{1}{c}{151.7}         & \multicolumn{1}{c}{238.7}        & \multicolumn{1}{c}{234000}         & 975           & \multicolumn{1}{c}{0}            & \multicolumn{1}{c}{9135}         & \multicolumn{1}{c}{206.1}         & \multicolumn{1}{c}{141.3}        & \multicolumn{1}{c}{101760}         & 424           \\ \hline
    \multirow{2}{*}{Variables} & \multicolumn{6}{c}{\textbf{Pump Replacing}}                                                                                                                                                         & \multicolumn{6}{c}{\textbf{Fracturing}}                                                                                                                                                             \\ \cline{2-13} 
                                     & \multicolumn{1}{c}{\textbf{min}} & \multicolumn{1}{c}{\textbf{max}} & \multicolumn{1}{c}{\textbf{mean}} & \multicolumn{1}{c}{\textbf{std}} & \multicolumn{1}{c}{\textbf{count}} & \multicolumn{1}{c|}{\textbf{num}} & \multicolumn{1}{c}{\textbf{min}} & \multicolumn{1}{c}{\textbf{max}} & \multicolumn{1}{c}{\textbf{mean}} & \multicolumn{1}{c}{\textbf{std}} & \multicolumn{1}{c}{\textbf{count}} & \multicolumn{1}{c}{\textbf{num}} \\ \hline
    $\mathbf{X}_{1}$                 & \multicolumn{1}{c}{0}            & \multicolumn{1}{c}{31}           & \multicolumn{1}{c}{25}            & \multicolumn{1}{c}{10}           & \multicolumn{1}{c}{4560}           & 19            & \multicolumn{1}{c}{0}            & \multicolumn{1}{c}{31}           & \multicolumn{1}{c}{22}            & \multicolumn{1}{c}{12}           & \multicolumn{1}{c}{13440}          & 56            \\ 
    $\mathbf{X}_{2}$                 & \multicolumn{1}{c}{0}            & \multicolumn{1}{c}{54.8}         & \multicolumn{1}{c}{12.73}         & \multicolumn{1}{c}{13.23}        & \multicolumn{1}{c}{4560}           & 19            & \multicolumn{1}{c}{0}            & \multicolumn{1}{c}{133.2}        & \multicolumn{1}{c}{15.74}         & \multicolumn{1}{c}{18.7}         & \multicolumn{1}{c}{13440}          & 56            \\ 
    $\mathbf{X}_{3}$                 & \multicolumn{1}{c}{100}          & \multicolumn{1}{c}{2938}         & \multicolumn{1}{c}{1007}          & \multicolumn{1}{c}{354.0}        & \multicolumn{1}{c}{4560}           & 19            & \multicolumn{1}{c}{0}            & \multicolumn{1}{c}{2294}         & \multicolumn{1}{c}{1694.3}        & \multicolumn{1}{c}{285.2}        & \multicolumn{1}{c}{13440}          & 56            \\ 
    $\mathbf{X}_{4}$                 & \multicolumn{1}{c}{0}            & \multicolumn{1}{c}{200}          & \multicolumn{1}{c}{66.04}         & \multicolumn{1}{c}{21.91}        & \multicolumn{1}{c}{4560}           & 19            & \multicolumn{1}{c}{0}            & \multicolumn{1}{c}{83}           & \multicolumn{1}{c}{45.01}         & \multicolumn{1}{c}{7.11}         & \multicolumn{1}{c}{13440}          & 56            \\ 
    $\mathbf{X}_{5}$                 & \multicolumn{1}{c}{0}            & \multicolumn{1}{c}{603.7}        & \multicolumn{1}{c}{51.95}         & \multicolumn{1}{c}{51.49}        & \multicolumn{1}{c}{4560}           & 19            & \multicolumn{1}{c}{0}            & \multicolumn{1}{c}{4981}         & \multicolumn{1}{c}{31.48}         & \multicolumn{1}{c}{148.8}        & \multicolumn{1}{c}{13440}          & 56            \\ 
    $\mathbf{X}_{6}$                 & \multicolumn{1}{c}{0}            & \multicolumn{1}{c}{275.4}        & \multicolumn{1}{c}{88.97}         & \multicolumn{1}{c}{55.40}        & \multicolumn{1}{c}{4560}           & 19            & \multicolumn{1}{c}{0}            & \multicolumn{1}{c}{433.3}        & \multicolumn{1}{c}{28.62}         & \multicolumn{1}{c}{25.82}        & \multicolumn{1}{c}{13440}          & 56            \\ 
    $\mathbf{X}_{7}$                 & \multicolumn{1}{c}{1}            & \multicolumn{1}{c}{4981}         & \multicolumn{1}{c}{578.1}         & \multicolumn{1}{c}{414}          & \multicolumn{1}{c}{4560}           & 19            & \multicolumn{1}{c}{2}            & \multicolumn{1}{c}{9107}         & \multicolumn{1}{c}{1226.8}        & \multicolumn{1}{c}{546.3}        & \multicolumn{1}{c}{13440}          & 56            \\ 
    $\mathbf{X}_{8}$                 & \multicolumn{1}{c}{0.6}          & \multicolumn{1}{c}{7.4}          & \multicolumn{1}{c}{3.39}          & \multicolumn{1}{c}{1.01}         & \multicolumn{1}{c}{4560}           & 19            & \multicolumn{1}{c}{0}            & \multicolumn{1}{c}{7}            & \multicolumn{1}{c}{3.66}          & \multicolumn{1}{c}{1.29}         & \multicolumn{1}{c}{13440}          & 56            \\ 
    $\mathbf{X}_{9}$                 & \multicolumn{1}{c}{0}            & \multicolumn{1}{c}{120}          & \multicolumn{1}{c}{9.06}          & \multicolumn{1}{c}{16.5}         & \multicolumn{1}{c}{4560}           & 19            & \multicolumn{1}{c}{0}            & \multicolumn{1}{c}{73}           & \multicolumn{1}{c}{6.34}          & \multicolumn{1}{c}{10.04}        & \multicolumn{1}{c}{13440}          & 56            \\ 
    $\mathbf{X}_{10}$                & \multicolumn{1}{c}{0}            & \multicolumn{1}{c}{7.1}          & \multicolumn{1}{c}{0.84}          & \multicolumn{1}{c}{1}            & \multicolumn{1}{c}{4560}           & 19            & \multicolumn{1}{c}{0}            & \multicolumn{1}{c}{60}           & \multicolumn{1}{c}{1.59}          & \multicolumn{1}{c}{3.22}         & \multicolumn{1}{c}{13440}          & 56            \\ 
    $\mathbf{Y}$                     & \multicolumn{1}{c}{0}            & \multicolumn{1}{c}{4112}         & \multicolumn{1}{c}{300.2}         & \multicolumn{1}{c}{367.7}        & \multicolumn{1}{c}{4560}           & 19            & \multicolumn{1}{c}{0}            & \multicolumn{1}{c}{3237}         & \multicolumn{1}{c}{108.14}        & \multicolumn{1}{c}{126.9}        & \multicolumn{1}{c}{13440}          & 56            \\ 
    \toprule
    \end{tabular}
\end{table*}

\subsection{Inside-well Time-series Forecasting}
In this section, we mainly study the performance of temporal models in inside-wells time-series forecasting,
which results are shown in Table \ref{tab:inside}.

\noindent \textbf{Experimental setup}.
In the phase of inside-well time-sereis forecasting, 
we characterize the last $\tau=36, 24, 12, 6$ monthly time-series of each well 
as the target domain $\mathcal{D}_\mathcal{T}$ respectively,
while the remained monthly time-series are formulized as source domain $\mathcal{D}_\mathcal{S}$. 
The experiments consist of 4 independent moudles,
corresponding to the forecasting task involving 
\textit{Sand Controlling, Perforation Adding, Pump Replacing} and \textit{Fracturing}, respectively.
For example, in the forecasting task of \textit{Sand Controlling} with $\tau=36$,
the knowledge of last 36 monthly time-series involving \textit{Sand Controlling} described in Table \ref{tab:datasets} 
is applied to transferred to the target domain $\mathcal{D}_\mathcal{T}$, 
i.e. conducting the inside-well time-sereis forecasting of oil production.

The results in Table \ref{tab:inside}
demonstrate that the performance of CDA is better than or on par with the baselines under the different 
experimental apparatus. 
We note the following observations to provide a better understanding into inside-well time-series forecasting. 
First, for any forecasting tasks,
the cross-domain forecasters (CDA, TFT, Transformer, DANN, MMDA and SASA) that jointly trained 
end-to-end using source and target domains outperform 
the single-domain forecasters (DecoderMLP, LSTM, GRU, Bi-LSTM, Bi-GRU, DeepAR and N-Beats). 
The finding indicates that jointly training on source and target domain is helpful for inside-well 
time-series forecasting. 
Second, 
in the forecasting tasks of \textit{Sand Controlling, Perforation Adding, Pump Replacing} and \textit{Fracturing},
whatever with $\tau=36, \tau=24, \tau=12$ or $\tau=6$, 
TFT, Transformer and DANN, MMDA and SASA are outperformed by CDA. 
The finding illustrates that time-series forecaster under casualilty and policy information is beneficial to achieve 
the better domain adaption in inside-well time-series forecasting. 

\begin{table*}[!t]
    \centering
    \caption{Performance comparison of multiple source domain knowledge in oil production forecasting of inside-well.}
    \label{tab:inside}
    \begin{tabular}{ccccccccccccc}
    \bottomrule
    \multicolumn{13}{c}{\textbf{Sand Controlling}}                                  \\ \hline
    \multicolumn{1}{c|}{\multirow{2}{*}{\textbf{Method}}} & \multicolumn{3}{c|}{\bm{$\tau=36$}}                                                                        & \multicolumn{3}{c|}{\bm{$\tau=24$}}      & \multicolumn{3}{c|}{\bm{$\tau=12$}}     & \multicolumn{3}{c}{\bm{$\tau=6$}}  \\ \cline{2-13} 
    \multicolumn{1}{c|}{}                                 & {\bm{$R^{2}$}} & {\textbf{RMSE}} & \multicolumn{1}{c|}{\textbf{MAE}} & {\bm{$R^{2}$}}  & {\textbf{RMSE}} & \multicolumn{1}{c|}{\textbf{MAE}} & {\bm{$R^{2}$}}  & {\textbf{RMSE}} & \multicolumn{1}{c|}{\textbf{MAE}} & {\bm{$R^{2}$}}  & {\textbf{RMSE}} & \multicolumn{1}{c}{\textbf{MAE}} \\ \hline
    \multicolumn{1}{c|}{DecoderMLP}                       & {0.219}            & {30.545}        & \multicolumn{1}{c|}{15.592}       & {0.300}            & {28.843}        & \multicolumn{1}{c|}{15.170}       & {0.396}       & {26.370}        & \multicolumn{1}{c|}{15.035}       & {0.479}          & {22.315}            & {18.961}            \\ 
    \multicolumn{1}{c|}{LSTM}                             & {0.240}            & {29.665}        & \multicolumn{1}{c|}{23.089}       & {0.827}            & {14.671}        & \multicolumn{1}{c|}{6.820}        & {0.696}       & {19.848}        & \multicolumn{1}{c|}{11.365}       & {0.751}          & {15.240}            & {12.365}           \\ 
    \multicolumn{1}{c|}{GRU}                              & {0.547}            & {22.897}        & \multicolumn{1}{c|}{13.650}       & {0.644}            & {20.419}        & \multicolumn{1}{c|}{10.783}       & {0.724}       & {18.558}        & \multicolumn{1}{c|}{13.953}       & {0.762}          & {15.078}            & {13.126}           \\ 
    \multicolumn{1}{c|}{Bi-LSTM}                          & {0.236}            & {29.861}        & \multicolumn{1}{c|}{21.059}       & {0.595}            & {22.135}        & \multicolumn{1}{c|}{11.235}       & {0.799}       & {15.574}        & \multicolumn{1}{c|}{10.182}       & {0.749}          & {15.680}            & {12.954}           \\ 
    \multicolumn{1}{c|}{Bi-GRU}                           & {0.597}            & {21.467}        & \multicolumn{1}{c|}{13.585}       & {0.782}            & {15.694}        & \multicolumn{1}{c|}{9.889}        & {0.807}       & {14.556}        & \multicolumn{1}{c|}{6.991}        & {0.819}          & {14.891}            & {12.892}          \\ 
    \multicolumn{1}{c|}{DeepAR}                           & {0.742}            & {18.023}        & \multicolumn{1}{c|}{13.621}       & {0.770}            & {16.621}        & \multicolumn{1}{c|}{10.448}       & {0.812}       & {13.997}        & \multicolumn{1}{c|}{6.630}        & {0.838}          & {13.958}            & {10.743}           \\ 
    \multicolumn{1}{c|}{N-Beats}                          & {0.506}            & {24.741}        & \multicolumn{1}{c|}{12.770}       & {0.733}            & {18.122}        & \multicolumn{1}{c|}{10.038}       & {0.873}       & {12.408}        & \multicolumn{1}{c|}{6.078}        & {0.886}          & {11.122}            & {9.751}           \\ 
    \multicolumn{1}{c|}{Transformer}                      & {0.826}            & {14.565}        & \multicolumn{1}{c|}{9.265}        & {0.839}            & {14.195}        & \multicolumn{1}{c|}{8.782}        & {0.817}       & {13.835}        & \multicolumn{1}{c|}{6.271}        & {0.895}          & {10.742}            & {8.378}           \\ 
    \multicolumn{1}{c|}{DANN}                             & {0.631}            & {20.578}        & \multicolumn{1}{c|}{12.041}       & {0.795}            & {15.178}        & \multicolumn{1}{c|}{9.471}        & {0.837}       & {14.058}        & \multicolumn{1}{c|}{7.648}        & {0.892}          & {10.342}            & {7.971}           \\ 
    \multicolumn{1}{c|}{SASA}                             & {0.672}            & {20.280}        & \multicolumn{1}{c|}{11.612}       & {0.821}            & {14.606}        & \multicolumn{1}{c|}{9.190}        & {0.866}       & {12.523}        & \multicolumn{1}{c|}{6.146}        & {0.911}          & {8.464}             & {6.716}           \\ 
    \multicolumn{1}{c|}{TFT}    & \textbf{0.869}    & \textbf{12.554}        & \multicolumn{1}{c|}{\textbf{7.554}}        & {0.913}            & {9.293}         & \multicolumn{1}{c|}{8.583}        & {0.923}       & {8.763}         & \multicolumn{1}{c|}{5.675}        & {0.935}          & {6.884}             & {4.559}           \\ 
    \multicolumn{1}{c|}{MMDA}                             & {0.769}            & {17.481}        & \multicolumn{1}{c|}{13.304}       & {0.895}            & {10.692}        & \multicolumn{1}{c|}{9.102}        & {0.910}       & {9.136}         & \multicolumn{1}{c|}{5.888}        & {0.940}          & {6.430}             & {4.252}           \\ 
    \multicolumn{1}{c|}{Ours}                             & {0.865}            & {12.641}        & \multicolumn{1}{c|}{7.676}  & \textbf{0.925}   & \textbf{8.741}         & \multicolumn{1}{c|}{\textbf{8.067}}   & \textbf{0.935} & \textbf{7.841}  & \multicolumn{1}{c|}{\textbf{5.150}}  & \textbf{0.949} & \textbf{5.638}  & \textbf{3.908}           \\ \hline
    \multicolumn{13}{c}{\textbf{Perforation Adding}}                                     \\ \hline
    \multicolumn{1}{c|}{\multirow{2}{*}{\textbf{Method}}} & \multicolumn{3}{c|}{\bm{$\tau=36$}}                                                                        & \multicolumn{3}{c|}{\bm{$\tau=24$}}      & \multicolumn{3}{c|}{\bm{$\tau=12$}}     & \multicolumn{3}{c}{\bm{$\tau=6$}}  \\ \cline{2-13}  
    \multicolumn{1}{c|}{}                                 & {\bm{$R^{2}$}}  & {\textbf{RMSE}} & \multicolumn{1}{c|}{\textbf{MAE}} & {\bm{$R^{2}$}}  & {\textbf{RMSE}} & \multicolumn{1}{c|}{\textbf{MAE}} & {\bm{$R^{2}$}}  & {\textbf{RMSE}} & \multicolumn{1}{c|}{\textbf{MAE}} & {\bm{$R^{2}$}}  & {\textbf{RMSE}} & \multicolumn{1}{c}{\textbf{MAE}}\\ \hline
    \multicolumn{1}{c|}{DecoderMLP}                       & {0.341}            & {43.089}        & \multicolumn{1}{c|}{35.202}       & {0.422}            & {35.769}        & \multicolumn{1}{c|}{31.992}        & {0.472}      & {34.203}        & \multicolumn{1}{c|}{30.699}       & {0.482}          & {34.787}             & {30.164}            \\ 
    \multicolumn{1}{c|}{LSTM}                             & {0.449}            & {35.260}        & \multicolumn{1}{c|}{30.290}       & {0.512}            & {31.004}        & \multicolumn{1}{c|}{26.979}        & {0.492}      & {33.642}        & \multicolumn{1}{c|}{29.677}       & {0.495}          & {33.208}             & {28.907}            \\ 
    \multicolumn{1}{c|}{GRU}                              & {0.487}            & {33.168}        & \multicolumn{1}{c|}{28.587}       & {0.537}            & {30.050}        & \multicolumn{1}{c|}{26.685}        & {0.526}      & {30.942}        & \multicolumn{1}{c|}{27.754}       & {0.580}          & {28.181}             & {22.819}           \\ 
    \multicolumn{1}{c|}{Bi-LSTM}                          & {0.399}            & {40.308}        & \multicolumn{1}{c|}{33.890}       & {0.494}            & {32.800}        & \multicolumn{1}{c|}{29.764}        & {0.516}      & {31.013}        & \multicolumn{1}{c|}{28.983}       & {0.593}          & {27.677}             & {21.262}            \\ 
    \multicolumn{1}{c|}{Bi-GRU}                           & {0.561}            & {28.334}        & \multicolumn{1}{c|}{23.887}       & {0.498}            & {32.136}        & \multicolumn{1}{c|}{28.457}        & {0.650}      & {25.501}        & \multicolumn{1}{c|}{21.330}       & {0.573}          & {29.432}             & {23.917}            \\ 
    \multicolumn{1}{c|}{DeepAR}                           & {0.583}            & {26.103}        & \multicolumn{1}{c|}{20.818}       & {0.629}            & {26.153}        & \multicolumn{1}{c|}{23.086}        & {0.671}      & {27.128}        & \multicolumn{1}{c|}{22.279}       & {0.724}          & {22.722}             & {17.889}            \\ 
    \multicolumn{1}{c|}{N-Beats}                          & {0.728}            & {18.353}        & \multicolumn{1}{c|}{14.203}       & {0.787}            & {21.526}        & \multicolumn{1}{c|}{17.402}        & {0.717}      & {24.544}        & \multicolumn{1}{c|}{19.369}       & {0.841}          & {14.785}             & {11.215}            \\ 
    \multicolumn{1}{c|}{Transformer}                      & {0.754}            & {17.035}        & \multicolumn{1}{c|}{13.469}       & {0.805}            & {19.028}        & \multicolumn{1}{c|}{15.297}        & {0.835}      & {16.818}        & \multicolumn{1}{c|}{12.704}       & {0.886}          & {9.510}              & {5.902}           \\ 
    \multicolumn{1}{c|}{DANN}                             & {0.652}            & {22.058}        & \multicolumn{1}{c|}{17.515}       & {0.670}            & {25.482}        & \multicolumn{1}{c|}{21.198}        & {0.863}      & {12.715}        & \multicolumn{1}{c|}{9.426}        & {0.874}          & {10.435}             & {6.356}           \\ 
    \multicolumn{1}{c|}{SASA}                             & {0.821}            & {14.719}        & \multicolumn{1}{c|}{10.658}       & {0.843}            & {14.814}        & \multicolumn{1}{c|}{11.677}        & {0.829}      & {17.255}        & \multicolumn{1}{c|}{13.522}       & {0.891}          & {8.175}              & {4.781}           \\ 
    \multicolumn{1}{c|}{TFT}                              & {0.850}            & {13.714}        & \multicolumn{1}{c|}{10.722}       & {0.867}            & {11.657}        & \multicolumn{1}{c|}{8.397}        & {0.884}      & {9.832}        & \multicolumn{1}{c|}{6.299}        & {0.906}          & {7.263}             & {4.159}           \\ 
    \multicolumn{1}{c|}{MMDA}                             & {0.836}            & {14.159}        & \multicolumn{1}{c|}{11.170}       & {0.851}            & {13.474}        & \multicolumn{1}{c|}{9.422}         & {0.868}      & {12.327}        & \multicolumn{1}{c|}{8.155}        & {0.915}          & {6.043}              & {3.980}           \\ 
    \multicolumn{1}{c|}{Ours}   & \textbf{0.853}  & \textbf{13.270}  & \multicolumn{1}{c|}{\textbf{10.091}}  & \textbf{0.872}   & \textbf{11.292}  & \multicolumn{1}{c|}{\textbf{8.086}}   & \textbf{0.897} & \textbf{8.539} & \multicolumn{1}{c|}{\textbf{5.041}}   & \textbf{0.924}     & \textbf{5.121}     & \textbf{2.501}           \\ \hline
    \multicolumn{13}{c}{\textbf{Pump Replacing}}       \\ \hline
    \multicolumn{1}{c|}{\multirow{2}{*}{\textbf{Method}}} & \multicolumn{3}{c|}{\bm{$\tau=36$}}                                                                        & \multicolumn{3}{c|}{\bm{$\tau=24$}}      & \multicolumn{3}{c|}{\bm{$\tau=12$}}     & \multicolumn{3}{c}{\bm{$\tau=6$}}  \\ \cline{2-13} 
    \multicolumn{1}{c|}{}                                 & {\bm{$R^{2}$}}  & {\textbf{RMSE}} & \multicolumn{1}{c|}{\textbf{MAE}} & {\bm{$R^{2}$}}  & {\textbf{RMSE}} & \multicolumn{1}{c|}{\textbf{MAE}} & {\bm{$R^{2}$}}  & {\textbf{RMSE}} & \multicolumn{1}{c|}{\textbf{MAE}} & {\bm{$R^{2}$}}  & {\textbf{RMSE}} & \multicolumn{1}{c}{\textbf{MAE}}\\ \hline
    \multicolumn{1}{c|}{DecoderMLP}                       & {0.252}            & {28.944}       & \multicolumn{1}{c|}{20.146}           & {0.304}         & {26.093}        & \multicolumn{1}{c|}{17.620}           & {0.405}          & {21.689}           & \multicolumn{1}{c|}{10.905}           & {0.426}           & {16.684}           & {8.739}           \\ 
    \multicolumn{1}{c|}{LSTM}                             & {0.281}            & {27.256}       & \multicolumn{1}{c|}{19.884}           & {0.323}         & {25.731}        & \multicolumn{1}{c|}{16.557}           & {0.458}          & {19.321}           & \multicolumn{1}{c|}{9.611}            & {0.481}           & {15.648}           & {7.510}           \\ 
    \multicolumn{1}{c|}{GRU}                              & {0.435}            & {22.638}       & \multicolumn{1}{c|}{15.154}           & {0.481}         & {22.203}        & \multicolumn{1}{c|}{14.883}           & {0.454}          & {19.745}           & \multicolumn{1}{c|}{9.891}            & {0.502}           & {14.423}           & {7.132}            \\ 
    \multicolumn{1}{c|}{Bi-LSTM}                          & {0.430}            & {22.852}       & \multicolumn{1}{c|}{15.443}           & {0.515}         & {21.009}        & \multicolumn{1}{c|}{13.275}           & {0.550}          & {16.118}           & \multicolumn{1}{c|}{9.214}            & {0.564}           & {12.220}           & {6.528}            \\ 
    \multicolumn{1}{c|}{Bi-GRU}                           & {0.451}            & {22.128}       & \multicolumn{1}{c|}{15.055}           & {0.529}         & {20.988}        & \multicolumn{1}{c|}{12.987}           & {0.563}          & {15.963}           & \multicolumn{1}{c|}{8.871}            & {0.598}           & {11.780}           & {6.186}            \\ 
    \multicolumn{1}{c|}{DeepAR}                           & {0.517}            & {20.882}       & \multicolumn{1}{c|}{14.819}           & {0.578}         & {19.869}        & \multicolumn{1}{c|}{12.866}           & {0.595}          & {15.307}           & \multicolumn{1}{c|}{8.214}            & {0.636}           & {9.921}            & {5.631}            \\ 
    \multicolumn{1}{c|}{N-Beats}                          & {0.590}            & {19.402}       & \multicolumn{1}{c|}{13.333}           & {0.682}         & {16.047}        & \multicolumn{1}{c|}{10.881}           & {0.720}          & {10.980}           & \multicolumn{1}{c|}{5.975}            & {0.779}           & {7.410}            & {4.371}           \\ 
    \multicolumn{1}{c|}{Transformer}                      & {0.633}            & {18.334}       & \multicolumn{1}{c|}{11.682}           & {0.698}         & {15.076}        & \multicolumn{1}{c|}{10.303}           & {0.732}          & {10.497}           & \multicolumn{1}{c|}{5.062}            & {0.807}           & {6.023}            & {3.374}            \\ 
    \multicolumn{1}{c|}{DANN}                             & {0.638}            & {17.191}       & \multicolumn{1}{c|}{10.692}           & {0.601}         & {18.798}        & \multicolumn{1}{c|}{11.667}           & {0.652}          & {12.922}           & \multicolumn{1}{c|}{7.182}            & {0.717}           & {7.925}            & {4.961}            \\ 
    \multicolumn{1}{c|}{SASA}                             & {0.610}            & {18.933}       & \multicolumn{1}{c|}{12.030}           & {0.810}         & {9.982}        & \multicolumn{1}{c|}{5.038}            & {0.765}          & {10.063}           & \multicolumn{1}{c|}{5.180}            & {0.814}           & {5.618}            & {3.002}            \\ 
    \multicolumn{1}{c|}{TFT}                              & {0.673}            & {15.350}       & \multicolumn{1}{c|}{7.757}            & {0.788}         & {12.093}        & \multicolumn{1}{c|}{6.135}            & {0.802}          & {8.621}            & \multicolumn{1}{c|}{3.756}            & {0.830}           & {5.333}            & {2.739}            \\ 
    \multicolumn{1}{c|}{MMDA}                             & {0.682}            & {15.814}       & \multicolumn{1}{c|}{7.933}            & {0.798}         & {11.281}        & \multicolumn{1}{c|}{6.983}            & {0.825}          & {7.201}            & \multicolumn{1}{c|}{3.728}            & {0.841}           & {4.626}            & {2.126}            \\ 
    \multicolumn{1}{c|}{Ours}  & \textbf{0.724}  & \textbf{13.159} & \multicolumn{1}{c|}{\textbf{5.111}}   & \textbf{0.813}  & \textbf{9.514}   & \multicolumn{1}{c|}{\textbf{4.461}}     & \textbf{0.836}     & \textbf{6.412}    & \multicolumn{1}{c|}{\textbf{2.156}}   & \textbf{0.859}   & \textbf{4.098}   & \textbf{1.822}           \\ \hline
    \multicolumn{13}{c}{\textbf{Fracturing}}    \\ \hline
    \multicolumn{1}{c|}{\multirow{2}{*}{\textbf{Method}}} & \multicolumn{3}{c|}{\bm{$\tau=36$}}                                                                        & \multicolumn{3}{c|}{\bm{$\tau=24$}}      & \multicolumn{3}{c|}{\bm{$\tau=12$}}     & \multicolumn{3}{c}{\bm{$\tau=6$}}  \\ \cline{2-13} 
    \multicolumn{1}{c|}{}                                 & {\bm{$R^{2}$}}  & {\textbf{RMSE}} & \multicolumn{1}{c|}{\textbf{MAE}} & {\bm{$R^{2}$}}  & {\textbf{RMSE}} & \multicolumn{1}{c|}{\textbf{MAE}} & {\bm{$R^{2}$}}  & {\textbf{RMSE}} & \multicolumn{1}{c|}{\textbf{MAE}} & {\bm{$R^{2}$}}  & {\textbf{RMSE}} & \multicolumn{1}{c}{\textbf{MAE}} \\ \hline
    \multicolumn{1}{c|}{DecoderMLP}                       & {0.316}            & {22.837}        & \multicolumn{1}{c|}{15.188}           & {0.389}          & {25.080}          & \multicolumn{1}{c|}{15.208}           & {0.453}       & {16.752}            & \multicolumn{1}{c|}{8.873}          & {0.497}         & {14.590}            & {9.806}          \\ 
    \multicolumn{1}{c|}{LSTM}                             & {0.384}            & {21.674}        & \multicolumn{1}{c|}{14.973}           & {0.423}          & {22.540}          & \multicolumn{1}{c|}{14.412}           & {0.503}       & {15.428}            & \multicolumn{1}{c|}{8.180}          & {0.568}         & {12.101}            & {8.704}           \\ 
    \multicolumn{1}{c|}{GRU}                              & {0.493}            & {17.593}        & \multicolumn{1}{c|}{10.629}           & {0.514}          & {19.187}          & \multicolumn{1}{c|}{13.605}           & {0.573}       & {14.215}            & \multicolumn{1}{c|}{7.845}          & {0.589}         & {11.287}            & {8.050}           \\ 
    \multicolumn{1}{c|}{Bi-LSTM}                          & {0.487}            & {18.037}        & \multicolumn{1}{c|}{11.509}           & {0.561}          & {17.979}          & \multicolumn{1}{c|}{11.273}           & {0.586}       & {13.357}            & \multicolumn{1}{c|}{7.236}          & {0.667}         & {9.717}             & {6.287}           \\ 
    \multicolumn{1}{c|}{Bi-GRU}                           & {0.533}            & {16.997}        & \multicolumn{1}{c|}{9.695}            & {0.592}          & {16.649}          & \multicolumn{1}{c|}{10.867}           & {0.677}       & {10.450}            & \multicolumn{1}{c|}{5.762}          & {0.696}         & {9.042}             & {6.127}            \\ 
    \multicolumn{1}{c|}{DeepAR}                           & {0.554}            & {16.329}        & \multicolumn{1}{c|}{9.314}            & {0.569}          & {15.087}          & \multicolumn{1}{c|}{10.050}           & {0.651}       & {11.032}            & \multicolumn{1}{c|}{6.388}          & {0.725}         & {7.297}             & {5.157}           \\ 
    \multicolumn{1}{c|}{N-Beats}                          & {0.683}            & {15.774}        & \multicolumn{1}{c|}{8.991}            & {0.736}          & {12.537}          & \multicolumn{1}{c|}{8.224}            & {0.801}       & {7.678}             & \multicolumn{1}{c|}{4.484}          & {0.827}         & {5.131}             & {3.331}             \\ 
    \multicolumn{1}{c|}{Transformer}                      & {0.749}            & {14.031}        & \multicolumn{1}{c|}{7.041}            & {0.808}          & {10.551}          & \multicolumn{1}{c|}{6.954}            & {0.816}       & {6.687}             & \multicolumn{1}{c|}{3.867}          & {0.837}         & {4.874}             & {2.981}           \\ 
    \multicolumn{1}{c|}{DANN}                             & {0.728}            & {14.835}        & \multicolumn{1}{c|}{7.238}            & {0.759}          & {11.552}          & \multicolumn{1}{c|}{7.760}            & {0.771}       & {9.180}             & \multicolumn{1}{c|}{5.349}          & {0.808}         & {6.439}             & {4.606}            \\ 
    \multicolumn{1}{c|}{SASA}                             & {0.769}            & {13.869}        & \multicolumn{1}{c|}{6.784}            & {0.816}          & {9.698}           & \multicolumn{1}{c|}{6.679}            & {0.845}       & {5.093}             & \multicolumn{1}{c|}{2.687}          & {0.859}         & {4.682}             & {2.848}           \\ 
    \multicolumn{1}{c|}{TFT}                              & {0.812}            & {11.811}        & \multicolumn{1}{c|}{5.973}            & {0.869}          & {7.853}           & \multicolumn{1}{c|}{4.704}            & {0.887}       & {5.059}             & \multicolumn{1}{c|}{2.476}          & {0.872}         & {4.135}             & {2.095}           \\ 
    \multicolumn{1}{c|}{MMDA}                             & {0.824}            & {11.241}        & \multicolumn{1}{c|}{5.477}            & {0.859}          & {7.680}           & \multicolumn{1}{c|}{4.679}            & {0.866}       & {5.340}             & \multicolumn{1}{c|}{2.340}          & {0.887}         & {3.447}             & {1.840}           \\ 
    \multicolumn{1}{c|}{Ours}  & \textbf{0.842}  & \textbf{11.003} & \multicolumn{1}{c|}{\textbf{4.307}}  & \textbf{0.870} & \textbf{7.004}   & \multicolumn{1}{c|}{\textbf{4.420}}  & \textbf{0.897}  & \textbf{4.362}  & \multicolumn{1}{c|}{\textbf{2.059}}  & \textbf{0.912}   & \textbf{3.073}   & \textbf{1.339}            \\ 
    \toprule
    \end{tabular}
\end{table*}

\subsection{Cross-well Time-sereis Predicting}
In this section, we perform the extensive experiments to compare the proposed CDA
with other models in cross-wells time-sereis predicting, 
which results are presented in Table \ref{tab:cross}.

\noindent \textbf{Experimental setup}.
In the phase of cross-well time-sereis predicting,
the entire time-series of partial wells are denoted as the source domain $\mathcal{D}_\mathcal{S}$, 
while the remained wells' entire time-series are remarked as target domain $\mathcal{D}_\mathcal{T}$.
The experiments also consist of 4 independent moudles,
corresponding to the forecasting task involving 
\textit{Sand Controlling, Perforation Adding, Pump Replacing} and \textit{Fracturing}, respectively.
For the instance, in the forecasting task of \textit{Sand Controlling},
the target domain $\mathcal{D}_\mathcal{T}$ is the dataset described in Table \ref{tab:datasets} 
involving \textit{Sand Controlling}. And for its source domain $\mathcal{D}_\mathcal{S}$, 
$\nsubseteq \mathbf{Z}$ represents that 
the source domain $\mathcal{D}_\mathcal{S}$ does not contain any knowledge of \textit{Sand Controlling},
while $\subseteq \mathbf{Z}$ represents that 
the certain knowledge of \textit{Sand Controlling} in source domain $\mathcal{D}_\mathcal{S}$ is
transferred to the target domain $\mathcal{D}_\mathcal{T}$, i.e.,
conducting the cross-well time-sereis predicting of oil production.

\begin{table*}[t!]
    \centering
    \caption{Performance comparison of multiple source domain knowledge in oil production forecasting of cross-well.}
    \label{tab:cross}
    \begin{tabular}{c|cccccc|cccccc}
    \bottomrule
    \textbf{Measure}                 & \multicolumn{6}{c|}{\textbf{Sand Controlling}}                                                                                                                                                             & \multicolumn{6}{c}{\textbf{Perforation Adding}}                                                                                                                                                           \\ \hline
    \multirow{2}{*}{\textbf{Method}} & \multicolumn{3}{c|}{$\nsubseteq \mathbf{Z}$}                                                                           & \multicolumn{3}{c|}{$\subseteq \mathbf{Z}$}                                                         & \multicolumn{3}{c|}{$\nsubseteq \mathbf{Z}$}               & \multicolumn{3}{c}{\textbf{$\subseteq \mathbf{Z}$}}   \\ \cline{2-13} 
                                     & {\bm{$R^{2}$}} & {\textbf{RMSE}} & \multicolumn{1}{c|}{\textbf{MAE}} & {\bm{$R^{2}$}} & {\textbf{RMSE}} & \textbf{MAE} & {\bm{$R^{2}$}} & {\textbf{RMSE}} & \multicolumn{1}{c|}{\textbf{MAE}} & {\bm{$R^{2}$}} & {\textbf{RMSE}} & \textbf{MAE} \\ \hline
    DecoderMLP                       & {0.372}            & {41.911}        & \multicolumn{1}{c|}{22.941}       & {0.829}         & {16.844}             & \multicolumn{1}{c|}{12.043}           & {0.528}         & {32.660}       & \multicolumn{1}{c|}{26.217}       & {0.824}       & {23.064}        & {15.340}           \\ 
    LSTM                             & {0.535}            & {33.500}        & \multicolumn{1}{c|}{13.888}       & {0.835}         & {16.483}             & \multicolumn{1}{c|}{11.971}           & {0.539}         & {30.957}       & \multicolumn{1}{c|}{25.731}       & {0.906}       & {16.580}        & {14.518}           \\ 
    GRU                              & {0.571}            & {31.367}        & \multicolumn{1}{c|}{16.357}       & {0.837}         & {15.267}             & \multicolumn{1}{c|}{11.187}           & {0.381}         & {43.327}       & \multicolumn{1}{c|}{35.742}       & {0.908}       & {16.638}        & {14.333}           \\ 
    Bi-LSTM                          & {0.448}            & {38.270}        & \multicolumn{1}{c|}{20.489}       & {0.844}         & {13.628}             & \multicolumn{1}{c|}{10.217}           & {0.580}         & {28.884}       & \multicolumn{1}{c|}{22.240}       & {0.850}       & {20.241}        & {11.275}           \\ 
    Bi-GRU                           & {0.556}            & {32.259}        & \multicolumn{1}{c|}{15.485}       & {0.862}         & {12.408}             & \multicolumn{1}{c|}{10.063}           & {0.563}         & {21.780}       & \multicolumn{1}{c|}{17.045}       & {0.859}       & {20.159}        & {11.111}           \\ 
    DeepAR                           & {0.849}            & {20.480}        & \multicolumn{1}{c|}{14.319}       & {0.881}         & {11.792}             & \multicolumn{1}{c|}{8.775}            & {0.584}         & {26.484}       & \multicolumn{1}{c|}{20.351}       & {0.912}       & {15.425}        & {10.002}           \\ 
    N-Beats                          & {0.775}            & {29.388}        & \multicolumn{1}{c|}{17.849}       & {0.923}         & {8.692}              & \multicolumn{1}{c|}{7.820}            & {0.774}         & {18.837}       & \multicolumn{1}{c|}{14.240}       & {0.945}       & {12.272}        & {8.980}            \\ 
    Transformer                      & {0.856}            & {19.716}        & \multicolumn{1}{c|}{13.689}       & {0.851}         & {9.938}              & \multicolumn{1}{c|}{7.830}            & {0.860}         & {16.270}       & \multicolumn{1}{c|}{21.992}       &  \textbf{0.953}   & {12.256}  &  \textbf{6.677}            \\ 
    DANN                             & {0.672}            & {24.988}        & \multicolumn{1}{c|}{12.881}       & {0.890}         & {10.214}             & \multicolumn{1}{c|}{7.772}            & {0.681}         & {20.729}       & \multicolumn{1}{c|}{17.144}       & {0.948}       & {12.071}        & {8.523}            \\ 
    SASA                             & {0.827}            & {22.947}        & \multicolumn{1}{c|}{9.738}        & {0.919}         & {8.858}              & \multicolumn{1}{c|}{6.080}            & {0.740}         & {18.860}       & \multicolumn{1}{c|}{14.455}       & {0.941}       & {13.645}        & {7.861}            \\ 
    TFT                              & {0.834}            & {21.456}        & \multicolumn{1}{c|}{8.972}        & {0.924}         & {8.866}              & \multicolumn{1}{c|}{6.052}            & {0.642}         & {24.622}       & \multicolumn{1}{c|}{20.690}       & {0.950}       & {12.596}        & {7.818}            \\ 
    MMDA                             & {0.861}            & {17.962}        & \multicolumn{1}{c|}{7.167}        & {0.930}         & {8.802}              & \multicolumn{1}{c|}{5.584}            & {0.832}         & {16.437}       & \multicolumn{1}{c|}{12.825}       & {0.952}       & {12.055}        & {7.529}            \\ 
    Ours                & \textbf{0.875} & \textbf{14.771} & \multicolumn{1}{c|}{\textbf{6.026}}  & \textbf{0.935}    & \textbf{7.694}   & \multicolumn{1}{c|}{\textbf{5.196}} & \textbf{0.886} & \textbf{16.110} & \multicolumn{1}{c|}{\textbf{12.366}}   & \textbf{0.953}  & \textbf{11.719}  & {6.679}            \\ \hline
    \textbf{Measure}                 & \multicolumn{6}{c|}{\textbf{Pump Replacing}}                                                                                                                                                 & \multicolumn{6}{c}{\textbf{Fracturing}}                                                                                                                                                                   \\ \hline
    \multirow{2}{*}{\textbf{Method}} & \multicolumn{3}{c|}{$\nsubseteq \mathbf{Z}$}                                                                           & \multicolumn{3}{c|}{$\subseteq \mathbf{Z}$}                                                         & \multicolumn{3}{c|}{$\nsubseteq \mathbf{Z}$}              & \multicolumn{3}{c}{$\subseteq \mathbf{Z}$}    \\ \cline{2-13} 
                                     & {\bm{$R^{2}$}} & {\textbf{RMSE}} & \multicolumn{1}{c|}{\textbf{MAE}} & {\bm{$R^{2}$}} & {\textbf{RMSE}} & \textbf{MAE} & {\bm{$R^{2}$}} & {\textbf{RMSE}} & \multicolumn{1}{c|}{\textbf{MAE}} & {\bm{$R^{2}$}} & {\textbf{RMSE}} & \textbf{MAE} \\ \hline
    DecoderMLP                       & {0.353}            & {38.545}        & \multicolumn{1}{c|}{30.199}           & {0.919}         & {13.381}             & \multicolumn{1}{c|}{6.753}        & {0.413}         & {23.760}       & \multicolumn{1}{c|}{17.329}           & {0.809}             & {11.405}        & {5.668}            \\ 
    LSTM                             & {0.471}            & {27.605}        & \multicolumn{1}{c|}{20.148}           & {0.947}         & {12.017}             & \multicolumn{1}{c|}{6.114}        & {0.501}         & {18.313}       & \multicolumn{1}{c|}{14.682}           & {0.937}             & {9.893}         & {4.674}            \\ 
    GRU                              & {0.484}            & {25.301}        & \multicolumn{1}{c|}{18.909}           & {0.941}         & {12.176}             & \multicolumn{1}{c|}{5.612}        & {0.518}         & {17.844}       & \multicolumn{1}{c|}{11.957}           & {0.955}             & {7.178}         & {3.462}            \\ 
    Bi-LSTM                          & {0.408}            & {33.401}        & \multicolumn{1}{c|}{24.153}           & {0.932}         & {11.325}             & \multicolumn{1}{c|}{5.366}        & {0.465}         & {20.884}       & \multicolumn{1}{c|}{14.257}           & {0.959}             & {7.803}         & {3.213}            \\ 
    Bi-GRU                           & {0.484}            & {25.587}        & \multicolumn{1}{c|}{17.779}           & {0.936}         & {11.357}             & \multicolumn{1}{c|}{4.818}        & {0.532}         & {17.287}       & \multicolumn{1}{c|}{12.891}           & {0.945}             & {9.441}         & {4.660}            \\ 
    DeepAR                           & {0.502}            & {23.173}        & \multicolumn{1}{c|}{15.129}           & {0.966}         & {10.684}             & \multicolumn{1}{c|}{4.804}        & {0.621}         & {16.005}       & \multicolumn{1}{c|}{11.707}           & {0.947}             & {9.265}         & {4.024}            \\ 
    N-Beats                          & {0.740}            & {18.839}        & \multicolumn{1}{c|}{12.650}           & {0.962}         & {9.205}              & \multicolumn{1}{c|}{3.914}        & {0.754}         & {14.062}       & \multicolumn{1}{c|}{9.211}            & {0.969}             & {7.697}         & {3.386}            \\ 
    Transformer                      & {0.781}            & {18.470}        & \multicolumn{1}{c|}{11.068}           & {0.962}         & {8.727}              & \multicolumn{1}{c|}{3.987}        & {0.765}         & {13.821}       & \multicolumn{1}{c|}{8.659}            & {0.966}             & {7.893}         & {3.414}            \\ 
    DANN                             & {0.639}            & {20.430}        & \multicolumn{1}{c|}{13.897}           & {0.968}         & {9.190}              & \multicolumn{1}{c|}{3.733}        & {0.681}         & {15.206}       & \multicolumn{1}{c|}{10.455}           & {0.974}             & {6.537}         & {2.029}            \\ 
    SASA                             & {0.751}            & {18.605}        & \multicolumn{1}{c|}{11.672}           & {0.969}         & {8.247}              & \multicolumn{1}{c|}{3.505}        & {0.787}         & {13.664}       & \multicolumn{1}{c|}{8.730}            & {0.976}             & {6.531}         & {1.841}            \\ 
    TFT                              & {0.774}            & {18.026}        & \multicolumn{1}{c|}{11.700}           & {0.971}    & \textbf{7.079}              & \multicolumn{1}{c|}{\textbf{3.113}}        & {0.819}         & {11.813}       & \multicolumn{1}{c|}{6.879}    & {0.970}          & {6.048}         & {1.736}            \\ 
    MMDA                             & {0.806}            & {16.366}        & \multicolumn{1}{c|}{9.162}            & {0.970}         & {7.754}              & \multicolumn{1}{c|}{3.144}        & {0.826}         & {11.655}       & \multicolumn{1}{c|}{6.545}            & {0.976}             & {5.593}         & {1.709}            \\ 
    Ours                 & \textbf{0.828} & \textbf{15.318}  & \multicolumn{1}{c|}{\textbf{8.467}}      & \textbf{0.971}    & {7.195}       & \multicolumn{1}{c|}{3.936}      & \textbf{0.836} & \textbf{10.944} & \multicolumn{1}{c|}{\textbf{5.917}}           & \textbf{0.981}          & \textbf{5.506}     & \textbf{1.477}            \\ 
    \toprule  
    \end{tabular}
\end{table*}

The results in Table \ref{tab:cross} also demonstrate the performance of CDA is 
better than or on par with the baselines under the corresponding experimental apparatus. 
There are also some observations that can provide a better understanding into cross-well time-series predicting. 
First, the cross-domain forecaster TFT, Transformer, MMDA, SASA and CDA outperform the single-domain forecaster 
DecoderMLP, LSTM, GRU, Bi-LSTM, Bi-GRU, DeepAR and N-Beats in two domain adaptation tasks. 
The finding indicates that jointly training on source and target domains is also helpful in cross-well 
time-series predicting. 
Second, the cross-domain forecaster, TFT, Transformer, DANN, and MMDA, 
are outperformed by CDA, indicating that domain adaptation combined with causal inferecnce is beneficial 
to achieve the better performance in cross-well time-series predicting. 
Third, the predicting performance of our CDA is better than or competitive to 
other models in industrial cross-well time-series predicting.

Significantly, the observations of Table \ref{tab:inside} and \ref{tab:cross} 
demonstrate the effectiveness of domain adaptation combined with causal inference, 
since it can capture the knowledge of policy-invariant representation under causality. 
Moreover, the results in Fig.\ref{figs:fits} illustrate that our method has achieved a better accuracy, 
especially in treatment region.
In oil and gas industry, the main purpose of applying policy is to improve oil production. 
Thus, based on intermediate result of counterfactual inference, Position-weise CATE in Definition \ref{def:CATE}, 
our proposed CDA can provide a guidance in determining a optimal policy for improving oil production. 
The interpretable results are displayed in following section.

\begin{figure*}[!t]
\centerline{\includegraphics[width=1\textwidth]{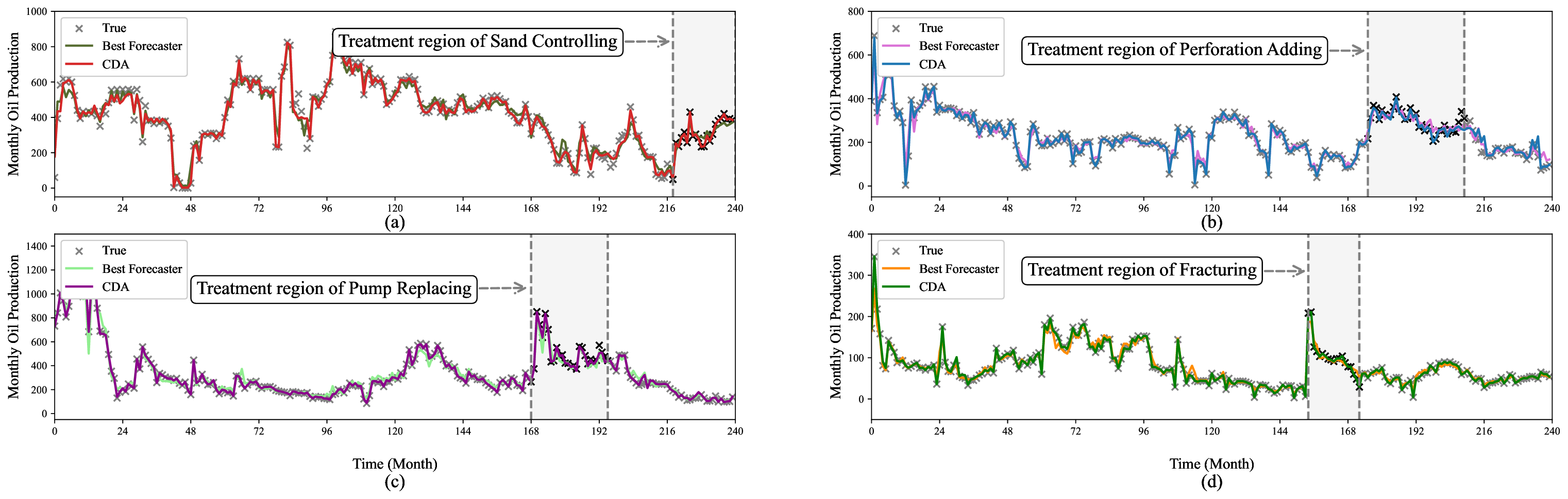}}
    \caption{Performance comparison of CFDA, best single-domain forecaster and cross-domain forecaster
    in estimating treatment-outcome}.
    \label{figs:fits}     
\end{figure*}

\subsection{Optimal Policy Determining}
In this section, we characterize our CDA on realizations of a synthetic decision-making process, 
mainly quantifying the causal effect of treatment policy $\mathbf{Z}$ impacting on oil production $\mathbf{Y}$, 
and determining the optimal policy.
To this end, we first estimate the outcome improvement that could have been achieved if treatment policies 
had been different to the observed ones in the treatment region, 
as dictated by the optimal treatment policy  and intervention.  
Then, we decode (hypothetical) future treatment policy to determine its latent state and use this to 
predict counterfactual outcome and proceed the optimal treatment policy for improving oil production.

\begin{figure*}[!t]
    \centerline{\includegraphics[scale=0.45]{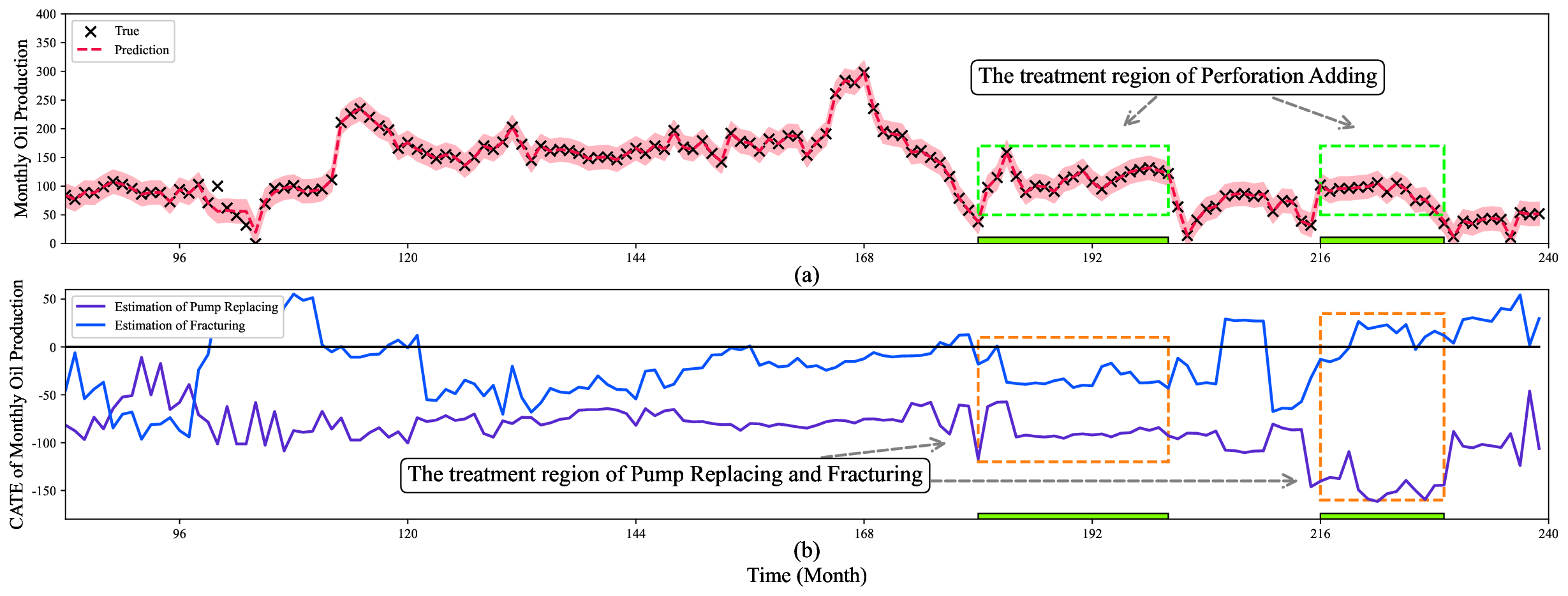}}
    \caption{
        The treatment region of policy and CATE of oil production.
        \textbf{Fig.6(a):} The treatment region under the policy \textit{Perforation Adding},
        where the points in green rectangles denote the response of oil production under the 
        policy \textit{Perforation Adding}.
        \textbf{Fig.6(b):}The CATE of oil production under the policies \textit{Pump Replacing} and \textit{Fracturing},
        where the points in orange rectangles denote the CATE of oil production under the 
        policies \textit{Pump Replacing} and \textit{Fracturing} and
        the black line denotes the CATE of actual policy \textit{Perforation Adding}.
    } 
    \label{figs:cate}     
\end{figure*}

As shown in the Top of Fig.\ref{figs:cate}, it summarizes the results 
for the CATE where an individual waterflooding well is intervened by different policy in treatment region
and for the guidance where CATE is utilized to decide whether the policy is reasonable or not. 
Compared with the ture oil prodution under the actual policy \textit{Perforation Adding} in Fig.\ref{figs:cate}(a),
there are some observations in Fig.\ref{figs:cate}(b) that 
can provide a better understanding into optimal policy in treatment region. 
First, since the the CATEs ($<0$) of \textit{Pump Replacing} and \textit{Fracturing} 
in first treatment region (orange rectangles) are both less than bias CATE 
($=0$, the CATE of actual policy \textit{Perforation Adding})
we can guess that the policies \textit{Pump Replacing} and \textit{Fracturing} are ineffective treatment.
for improving oil production in first treatment region.
Second, since the CATE ($<0$) of \textit{Pump Replacing} is less than bias CATE 
($=0$, the CATE of actual policy \textit{Perforation Adding})
and the CATE of \textit{Fracturing} is almost more than bias CATE 
($=0$, the CATE of actual policy \textit{Perforation Adding}) in second treatment region,
these results illustrate that the policies \textit{Pump Replacing} is ineffective treatment and 
the policies \textit{Fracturing} is more valuable treatment than policy \textit{Perforation Adding}
for improving oil production in second treatment region.

\begin{figure*}[!t]
    \centering
    \subfloat{\includegraphics[scale=0.4]{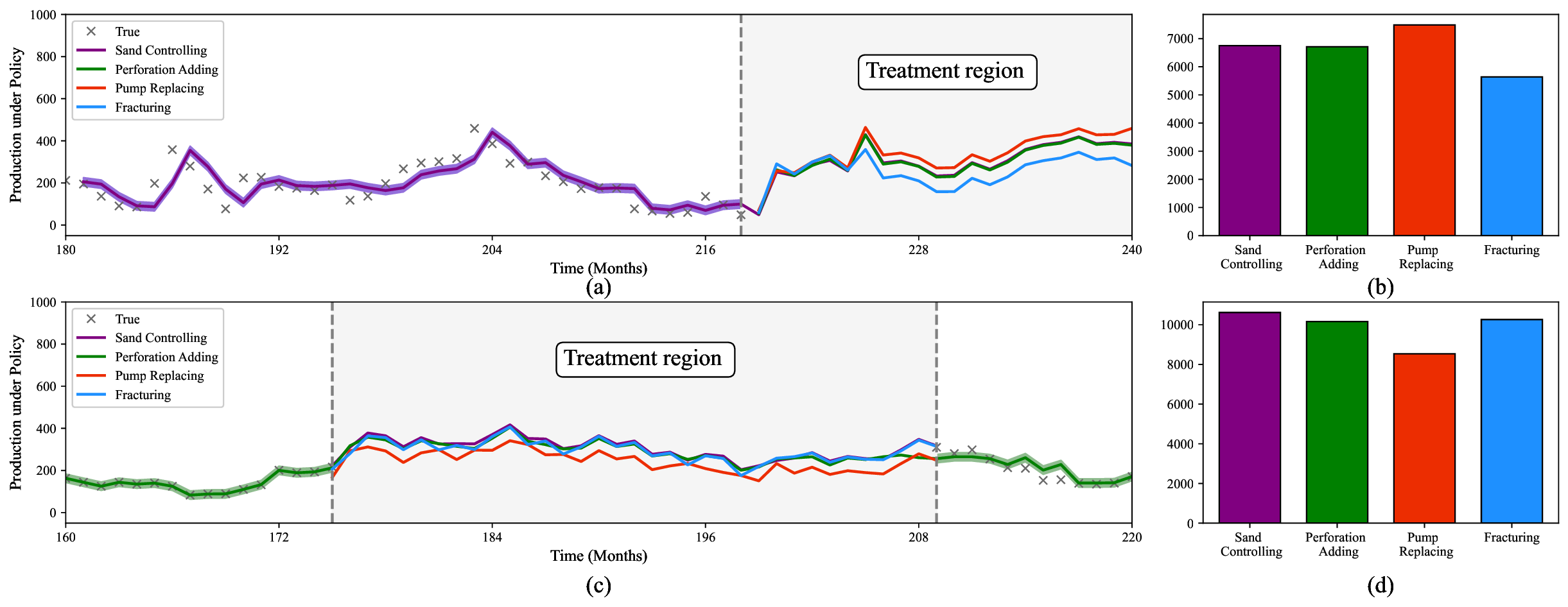}}
    \\[1pt]
    \subfloat{\includegraphics[scale=0.4]{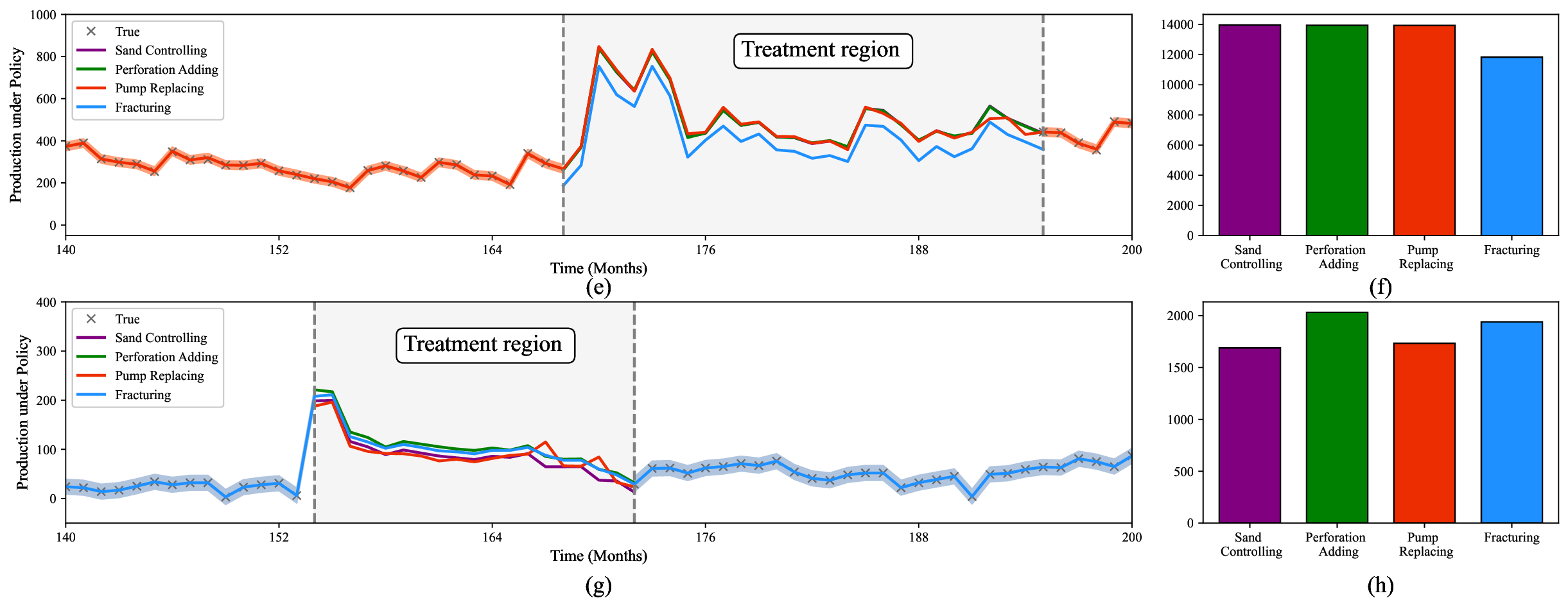}}
    \caption{Inside-Well Inference and Guidance for oil production under different treatment policy.
    \textbf{Left in Top:} The predicted oil production of Well 1 under treatment policy,
    where \textcolor{purple}{purple line} denote the actual oil production without any treatment policy.
    \textbf{Left in Bottom:} The predicted oil production of Well 2 under treatment policy,
    where \textcolor{purple}{purple line} denote the actual oil production without any treatment policy.
    \textbf{Right in Top:} The cumulative increment of oil production under treatment policy for Well 1.
    \textbf{Right in Bottom:} The cumulative increment of oil production under treatment policy for Well 2.}
    \label{figs:policy}     
\end{figure*}

In the left of Fig.\ref{figs:policy}, 
we provide the predicted production trajectories for the treatment policies \textit{Sucker Rod Pump} and
\textit{Perforation Adding}. And in the right of Fig.\ref{figs:ipolicy},
we provide the cumulative increment of predicted production (the response of the left Fig.\ref{figs:policy})
under the treatment policies \textit{Sucker Rod Pump} and \textit{Perforation Adding}.
Then, there are some guidance for improving of oil production. 
First, for Well 1, \textit{Sucker Rod Pump} has more significant and positive causal effects in achieving
the increment of oil production compared with \textit{Perforation Adding}, 
which is corresponding to the top of Fig.\ref{figs:inside-inference},
while the reverse hold false for Well 2.
Second, the predicted production and its CATE by CDA is an effective metric in guiding producing activity.

\section{Conclusion}
In this paper, we aim to combine domain adaptation with causal inferecnce to
industrial time-series forecasting to sovle the data scarcity and decision-confusing.
We identify the shared casuality over time in both domains,
which related theoretical analysis reveals the latent cause of domain shift 
from the perspective of casual inference, 
and accordingly propose a Causal Domain Adaptation (CDA) forecaster based on answer-based
attention mechanism, which is more line with industrial production process.
Different from the existing paradigm of time-series forecaster, CDA not only can perform
industrial time-series forecasting across domains even with different treatments

but also can provide a guidance in industrial production process along with treatments.
Through empirical experiments in real-word and synthetic oilfield datasets, 
we have demonstrated that CDA outperforms various single-domain and domain adaptation methods.
We further show that the effectiveness of our designs in answering casual queries 
about the interventional and counterfactual distribution of outcome in guiding industrial production process.

\bibliographystyle{IEEEtran}
\bibliography{Reference}

\clearpage
\onecolumn
\setcounter{theorem}{0}
\setcounter{corollary}{0}
\appendix
\noindent In this section, we give the proof of the Theorem \ref{the:the1}.
\begin{theorem}
    Let $\mu[P]$ be a distribution of $P$ in RKHS $\mathcal{H}_{k}$, 
    then via the reproducing property of RKHS $\mathcal{H}_{k}$, 
    we have 
    $\langle\phi, \mu[P_{\mathcal{T},\mathcal{S}}]\rangle=\mathbb{E}_{\mathbf{X}\sim\mathcal{T},\mathcal{S}}[\phi(\mathbf{X})]$
    for MMD,
    $\langle\phi, \mu[P_{\mathcal{T},\mathcal{S}|\mathbf{Z}}]\rangle=\mathbb{E}_{\mathbf{X}\sim\mathcal{T},\mathcal{S}}[\phi(\mathbf{X}|\mathbf{Z})]$
    for CMMD.
    Thus, the empirical estimate of squared CMMD shown in Definition \ref{def:cmmd} can be further simplifed as
    \begin{equation}
        \begin{aligned} 
            & \ \ \ \ d_{\mathrm{CMMD}}^2(\mathcal{S}|\mathbf{Z}, \mathcal{T}|\mathbf{Z}) \\
            & \leq \frac{1}{4}\left[d_{\mathrm{CMMD}}^2(\mathcal{S}|\mathbf{Z}, \mathcal{T}|\mathbf{Z})+d_{\mathrm{CMMD}}^2(\mathcal{S}, \mathcal{S}| \mathbf{Z}) + d_{\mathrm{CMMD}}^2(\mathcal{T}, \mathcal{T}| \mathbf{Z}) + d_{\mathrm{MMD}}^2(\mathcal{S}, \mathcal{T})+2d_{\mathrm{CMMD}}(\mathcal{S}|\mathbf{Z}, \mathcal{T}|\mathbf{Z})d_{\mathrm{MMD}}(\mathcal{S}, \mathcal{T})\right] \\
        \end{aligned} \nonumber
    \end{equation}
    where $d_{\mathrm{MMD}}$ is the standard Maximum Mean Distance (MMD),
    $d_{\mathrm{CMMD}}(\mathcal{S}, \mathcal{S}|\mathbf{Z})$ is applied to measure the closeness of source domain
    between the raw distributions $P_{\mathcal{S}}$ and condition distribution $P_{\mathcal{S}|\mathbf{Z}}$, and 
    $d_{\mathrm{CMMD}}(\mathcal{T}, \mathcal{T}|\mathbf{Z})$ is applied to measure the closeness of target domain
    between the raw distributions $P_{\mathcal{T}}$ and condition distribution $P_{\mathcal{T}|\mathbf{Z}}$.
\end{theorem}

\begin{proof}
    Given the distributions of source $P_{\mathcal{S}}$
    and target domain $P_{\mathcal{T}}$ \cite{gretton2012kernel, kumagai2019unsupervised}, 
    its MMD can be defined as
    \begin{equation}
        d_{\mathrm{MMD}}(\mathcal{S}, \mathcal{T})=\|\mu\left(P_{\mathcal{S}}\right)-\mu\left(P_{\mathcal{T}}\right)\|_{\mathcal{H}_k}
    \end{equation}  
    In addition, 
    since $\langle\phi, \mu[P_{\mathcal{T},\mathcal{S}}]\rangle=\mathbb{E}_{\mathbf{X}\sim\mathcal{T},\mathcal{S}}[\phi(\mathbf{X})]$
    in MMD and 
    $\langle\phi, \mu[P_{\mathcal{T},\mathcal{S}|\mathbf{Z}}]\rangle=\mathbb{E}_{\mathbf{X}\sim\mathcal{T},\mathcal{S}}[\phi(\mathbf{X}|\mathbf{Z})]$
    in CMMD, we have the following equations for CMMD
    \begin{equation}
        \begin{aligned}
            & d_{\mathrm{CMMD}}(\mathcal{S}| \mathbf{Z}, \mathcal{T} | \mathbf{Z})=\left\|\mu\left(P_{\mathcal{S} | Z}\right)-\mu\left(P_{\mathcal{T} | Z}\right)\right\|_{\mathcal{H}_k} \\
            & d_{\mathrm{CMMD}}(\mathcal{S}, \mathcal{S}| \mathbf{Z})=\|\mu\left(P_{\mathcal{S}}\right)-\mu\left(P_{\mathcal{S}|\mathbf{Z}}\right)\|_{\mathcal{H}_k} \\
            & d_{\mathrm{CMMD}}(\mathcal{T}, \mathcal{T}| \mathbf{Z})=\|\mu\left(P_{\mathcal{T}}\right)-\mu\left(P_{\mathcal{T}|\mathbf{Z}}\right)\|_{\mathcal{H}_k}
        \end{aligned}
    \end{equation}
    Thus, we can obtain
    \begin{equation}
        \begin{aligned} 
            & \ \ \ \ d_{\mathrm{CMMD}}(\mathcal{S}, \mathcal{T} | \mathbf{Z}) \\
            & =\left\|\mu\left(P_{\mathcal{S} | Z}\right)-\mu\left(P_{\mathcal{T} | Z}\right)\right\|_{\mathcal{H}_k} \\ 
            & =\left\|\mu\left(P_{\mathcal{S} | Z}\right)-\mu\left(P_{\mathcal{T} | Z}\right)-\mu\left(P_{\mathcal{S}}\right)+\mu\left(P_{\mathcal{T}}\right)+\mu\left(P_{\mathcal{S}}\right)-\mu\left(P_{\mathcal{T}}\right)\right\|_{\mathcal{H}_k} \\ 
            & =\left\|\left[\mu\left(P_{\mathcal{S} | Z}\right)-\mu\left(P_{\mathcal{S}}\right)\right]+\left[\mu\left(P_{\mathcal{S}}\right)-\mu\left(P_{\mathcal{T} | Z}\right)\right]+\left[\mu\left(P_{\mathcal{S}}\right)-\mu\left(P_{\mathcal{T}}\right)\right]\right\|_{\mathcal{H}_k} \\ 
            & \leq\left\|\mu\left(P_{\mathcal{S} | Z}\right)-\mu\left(P_{\mathcal{S}}\right)\right\|_{\mathcal{H}_k}+\left\|\mu\left(P_{\mathcal{T} | Z}\right)-\mu\left(P_{\mathcal{T}}\right)\right\|_{\mathcal{H}_k}+\left\|\left[\mu\left(P_{\mathcal{S}}\right)-\mu\left(P_{\mathcal{T}}\right)\right]\right\|_{\mathcal{H}_k}) \\
            & = d_{\mathrm{CMMD}}(\mathcal{S}, \mathcal{S}| \mathbf{Z}) + d_{\mathrm{CMMD}}(\mathcal{T}, \mathcal{T}| \mathbf{Z}) + d_{\mathrm{MMD}}(\mathcal{S}, \mathcal{T})
        \end{aligned} 
    \end{equation} 
    Let $d_{\mathrm{CMMD}}(\mathcal{S}, \mathcal{T} | \mathbf{Z}):=\sup_{\|\phi\|_{\mathcal{H}_{k}}\leq1}\left(d_{\mathrm{CMMD}}(\mathcal{S}, \mathcal{S}| \mathbf{Z}) + d_{\mathrm{CMMD}}(\mathcal{T}, \mathcal{T}| \mathbf{Z}) + d_{\mathrm{MMD}}(\mathcal{S}, \mathcal{T})\right)$,
    the nearly minimal estimation of CMMD can be represented as
    \begin{equation}
        \label{eqs:the1}
        \begin{aligned}
            & \ \ \ \ d_{\mathrm{CMMD}}^2(\mathcal{S}|\mathbf{Z}, \mathcal{T}|\mathbf{Z}) \\
            & =\frac{1}{4}\left[d_{\mathrm{CMMD}}(\mathcal{S}|\mathbf{Z}, \mathcal{T}|\mathbf{Z})+d_{\mathrm{CMMD}}(\mathcal{S}, \mathcal{S}| \mathbf{Z}) + d_{\mathrm{CMMD}}(\mathcal{T}, \mathcal{T}| \mathbf{Z}) + d_{\mathrm{MMD}}(\mathcal{S}, \mathcal{T})\right]^2 \\      
            & =\frac{1}{4}\left[d_{\mathrm{CMMD}}^2(\mathcal{S}|\mathbf{Z}, \mathcal{T}|\mathbf{Z})+d_{\mathrm{CMMD}}^2(\mathcal{S}, \mathcal{S}| \mathbf{Z}) + d_{\mathrm{CMMD}}^2(\mathcal{T}, \mathcal{T}| \mathbf{Z}) + d_{\mathrm{MMD}}^2(\mathcal{S}, \mathcal{T})\right] \\
            & + \frac{1}{4}\left[2d_{\mathrm{CMMD}}(\mathcal{S}|\mathbf{Z}, \mathcal{T}|\mathbf{Z})d_{\mathrm{CMMD}}(\mathcal{S}, \mathcal{S}| \mathbf{Z}) + 2d_{\mathrm{CMMD}}(\mathcal{S}|\mathbf{Z}, \mathcal{T}|\mathbf{Z})d_{\mathrm{CMMD}}(\mathcal{T}, \mathcal{T}| \mathbf{Z}) + 2d_{\mathrm{CMMD}}(\mathcal{S}|\mathbf{Z}, \mathcal{T}|\mathbf{Z})d_{\mathrm{MMD}}(\mathcal{S}, \mathcal{T})\right] \\
            & + \frac{1}{4}\left[2d_{\mathrm{CMMD}}(\mathcal{S}, \mathcal{S}|\mathbf{Z})d_{\mathrm{CMMD}}(\mathcal{T}, \mathcal{T}| \mathbf{Z}) + 2d_{\mathrm{CMMD}}(\mathcal{S}, \mathcal{S}|\mathbf{Z})d_{\mathrm{MMD}}(\mathcal{S}, \mathcal{T}) + 2d_{\mathrm{CMMD}}(\mathcal{T}, \mathcal{T}|\mathbf{Z})d_{\mathrm{MMD}}(\mathcal{S}, \mathcal{T})\right]
        \end{aligned}
    \end{equation}
    Since $d_{\mathrm{CMMD}}(\mathcal{S}, \mathcal{S}| \mathbf{Z})$ and $d_{\mathrm{CMMD}}(\mathcal{T}, \mathcal{T}| \mathbf{Z})$
    meausre the internal closeness of source $\mathcal{S}$ and target domain $\mathcal{T}$, respectively;
    $d_{\mathrm{CMMD}}(\mathcal{S}|\mathbf{Z}, \mathcal{T}|\mathbf{Z})$ and $d_{\mathrm{MMD}}(\mathcal{S}, \mathcal{T})$
    are both utilized to measure the closeness between source domain $\mathcal{S}$ and target domain $\mathcal{T}$,
    the mixed MMD is unvalid in dmain adversarial learning, that is, 
    \begin{equation}
        \begin{aligned}
            & d_{\mathrm{CMMD}}(\mathcal{S}|\mathbf{Z}, \mathcal{T}|\mathbf{Z})d_{\mathrm{CMMD}}(\mathcal{S}, \mathcal{S}| \mathbf{Z})=0\\
            & d_{\mathrm{CMMD}}(\mathcal{S}|\mathbf{Z}, \mathcal{T}|\mathbf{Z})d_{\mathrm{CMMD}}(\mathcal{T}, \mathcal{T}| \mathbf{Z})=0\\
            & d_{\mathrm{CMMD}}(\mathcal{S}, \mathcal{S}|\mathbf{Z})d_{\mathrm{CMMD}}(\mathcal{T}, \mathcal{T}| \mathbf{Z})=0\\
            & d_{\mathrm{CMMD}}(\mathcal{S}, \mathcal{S}|\mathbf{Z})d_{\mathrm{MMD}}(\mathcal{S}, \mathcal{T})=0\\
            & d_{\mathrm{CMMD}}(\mathcal{T}, \mathcal{T}|\mathbf{Z})d_{\mathrm{MMD}}(\mathcal{S}, \mathcal{T})=0.
        \end{aligned}
    \end{equation}
    Therefore, the $d_{\mathrm{CMMD}}^2(\mathcal{S}|\mathbf{Z}, \mathcal{T}|\mathbf{Z})$ satsifies
    \begin{equation}
        \begin{aligned} 
            & \ \ \ \ d_{\mathrm{CMMD}}^2(\mathcal{S}|\mathbf{Z}, \mathcal{T}|\mathbf{Z}) \\
            & \leq \frac{1}{4}\left[d_{\mathrm{CMMD}}^2(\mathcal{S}|\mathbf{Z}, \mathcal{T}|\mathbf{Z})+d_{\mathrm{CMMD}}^2(\mathcal{S}, \mathcal{S}| \mathbf{Z}) + d_{\mathrm{CMMD}}^2(\mathcal{T}, \mathcal{T}| \mathbf{Z}) + d_{\mathrm{MMD}}^2(\mathcal{S}, \mathcal{T})+2d_{\mathrm{CMMD}}(\mathcal{T}, \mathcal{T}|\mathbf{Z})d_{\mathrm{MMD}}(\mathcal{S}, \mathcal{T})\right] \\
        \end{aligned} \nonumber
    \end{equation}
\end{proof}
    
\begin{corollary}
    \label{cor:cor1}
    In domain adversarial learning, the MMD can be transformed into a more tractable form 
    from the perspective of loss function \cite{jin2022domain}, that is,
    \begin{equation}
        \begin{aligned}
            & \ \ \ \ d_{\mathrm{CMMD}}^2(\mathcal{S}|\mathbf{Z}, \mathcal{T}|\mathbf{Z}) \\
            & \Rightarrow \mathcal{L}(\mathcal{D}_{\mathcal{S}}, \mathcal{D}_{\mathcal{T}},\mathcal{D}_{\mathcal{S}|\mathbf{Z}}, \mathcal{D}_{\mathcal{T}|\mathbf{Z}};\mathcal{G}_{\mathcal{S}}, \mathcal{G}_{\mathcal{T}}) \\
            & 
            = \mathcal{L}_{1}(\mathcal{D}_{\mathcal{S}}, \mathcal{D}_{\mathcal{T}};\mathcal{G}_{\mathcal{S}}, \mathcal{G}_{\mathcal{T}})
            + \mathcal{L}_{2}(\mathcal{D}_{\mathcal{S}|\mathbf{Z}}, \mathcal{D}_{\mathcal{T}|\mathbf{Z}};\mathcal{G}_{\mathcal{S}}, \mathcal{G}_{\mathcal{T}}) \\
            & 
            + \mathcal{L}_{3}(\mathcal{D}_{\mathcal{S}},\mathcal{D}_{\mathcal{S}|\mathbf{Z}};\mathcal{G}_{\mathcal{S}})
            + \mathcal{L}_{4}(\mathcal{D}_{\mathcal{T}},\mathcal{D}_{\mathcal{T}|\mathbf{Z}};\mathcal{G}_{\mathcal{T}}) \\
            & 
            + \mathcal{L}_{1}(\mathcal{D}_{\mathcal{S}}, \mathcal{D}_{\mathcal{T}};\mathcal{G}_{\mathcal{S}}, \mathcal{G}_{\mathcal{T}})^{\frac{1}{2}}
            \mathcal{L}_{2}(\mathcal{D}_{\mathcal{S}|\mathbf{Z}}, \mathcal{D}_{\mathcal{T}|\mathbf{Z}};\mathcal{G}_{\mathcal{S}}, \mathcal{G}_{\mathcal{T}})^{\frac{1}{2}}
        \end{aligned} \nonumber
    \end{equation}
    where $\mathcal{D}_{\mathcal{S}}=\{\mathcal{S}, \mathcal{\hat{S}}\}, \mathcal{D}_{\mathcal{T}}=\{\mathcal{T}, \mathcal{\hat{T}}\}$, 
    $\mathcal{D}_{\mathcal{S}|\mathbf{Z}}=\{\mathcal{S}, \mathcal{\hat{S}}\}_{|\mathbf{Z}}$
    and $\mathcal{D}_{\mathcal{T}|\mathbf{Z}}=\{\mathcal{T}, \mathcal{\hat{T}}\}_{|\mathbf{Z}}$,
    which $\mathcal{\hat{S}}$ and $\mathcal{\hat{T}}$ 
    denote the generated source and target domains by sequence generators of source domain $\mathcal{G}_{\mathcal{S}}$
    and target domain $\mathcal{G}_{\mathcal{T}}$.
    \begin{equation}
        \left\{
        \begin{aligned}
        \mathcal{L}_{1}
        & =\beta_{1}\left\|\frac{1}{|\mathcal{D}_{\mathcal{S}}|}\sum_{\mathbf{X}\in{\mathcal{D}_{\mathcal{S}}}} \mathbf{X}
        -\frac{1}{|\mathcal{D}_{\mathcal{T}}|}\sum_{\mathbf{X}\in\mathcal{D}_{\mathcal{T}}} \mathbf{X}\right\|_{2}^{2} \\
        \mathcal{L}_{2}
        & =\beta_{2}\left\|\frac{1}{|\mathcal{D}_{\mathcal{S}}|}\sum_{\mathbf{a}\in\mathbf{Z}}\sum_{\mathbf{X}\in{\mathcal{D}_{\mathcal{S}|\mathbf{Z}}}} \mathbf{R}_{\mathbf{a}}
        -\frac{1}{|\mathcal{D}_{\mathcal{T}}|}\sum_{\mathbf{a}\in\mathbf{Z}}\sum_{\mathbf{X}\in\mathcal{D}_{\mathcal{T}|\mathbf{Z}}} \mathbf{R}_{\mathbf{a}}\right\|_{2}^{2}\\
        \mathcal{L}_{3}
        &=\beta_{3}\left\|\frac{1}{|\mathcal{D}_{\mathcal{S}}|}\sum_{\mathbf{X}\in{\mathcal{D}_{\mathcal{S}}}} \mathbf{X}
        -\frac{1}{|\mathcal{D}_{\mathcal{S}|\mathbf{Z}}|}\sum_{\mathbf{a}\in\mathbf{Z}}\sum_{\mathbf{X}\in\mathcal{D}_{\mathcal{S}|\mathbf{Z}}} \mathbf{R}_{\mathbf{a}}\right\|_{2}^{2}\\
        \mathcal{L}_{4}
        &=\beta_{4}\left\|\frac{1}{|\mathcal{D}_{\mathcal{T}}|}\sum_{\mathbf{X}\in{\mathcal{D}_{\mathcal{T}}}} \mathbf{X}
        -\frac{1}{|\mathcal{D}_{\mathcal{T}|\mathbf{Z}}|}\sum_{\mathbf{a}\in\mathbf{Z}}\sum_{\mathbf{X}\in\mathcal{D}_{\mathcal{T}|\mathbf{Z}}} \mathbf{R}_{\mathbf{a}}\right\|_{2}^{2}      
        \end{aligned} \nonumber
        \right.
    \end{equation}
    where $\mathbf{R}_{\mathbf{a}}:=\alpha(\mathbf{a}, \mathbf{k})\mathbf{X}$, 
    $\mathbf{k}$ denote the keys in transfer learning, 
    $|\cdot|$ is the cardinality function, 
    and the constants $\beta$ are the balance parameters.
\end{corollary}

\begin{proof}
    In domain adversarial learning, the MMD can be transformed into a more tractable form 
    by computing the distance between source and target domain \cite{jin2022domain}, that is,
    \begin{equation}
        \label{eqs:cor1}
        d_{\mathrm{MMD}}^2(\mathcal{S}, \mathcal{T})\Rightarrow \mathcal{L}(\mathcal{S}, \mathcal{T})
        =\left\|\frac{1}{|\mathcal{S}|}\sum_{\mathbf{X}\in{\mathcal{S}}} \mathbf{X}
        -\frac{1}{|\mathcal{T}|}\sum_{\mathbf{X}\in{\mathcal{T}}} \mathbf{X}\right\|_{2}^{2}
    \end{equation}
    Thus the independent loss functions of our proposed CDA can be formulized by utilizing Eq.(\ref{eqs:reconstruction-x}), 
    shown as
    \begin{equation}
        \label{eqs:cor2}
        \begin{aligned}
            & d_{\mathrm{MMD}}^2(\mathcal{S}, \mathcal{T})\Rightarrow\mathcal{L}_{1}(\mathcal{D}_{\mathcal{S}}, \mathcal{D}_{\mathcal{T}};\mathcal{G}_{\mathcal{S}}, \mathcal{G}_{\mathcal{T}}) 
            =\left\|\frac{1}{|\mathcal{D}_{\mathcal{S}}|}\sum_{\mathbf{X}\in{\mathcal{D}_{\mathcal{S}}}} \mathbf{X}
            -\frac{1}{|\mathcal{D}_{\mathcal{T}}|}\sum_{\mathbf{X}\in\mathcal{D}_{\mathcal{T}}} \mathbf{X}\right\|_{2}^{2}\\
            & d_{\mathrm{CMMD}}^2(\mathcal{S}|\mathbf{Z}, \mathcal{T}|\mathbf{Z})\Rightarrow\mathcal{L}_{2}(\mathcal{D}_{\mathcal{S}|\mathbf{Z}}, \mathcal{D}_{\mathcal{T}|\mathbf{Z}};\mathcal{G}_{\mathcal{S}}, \mathcal{G}_{\mathcal{T}}) 
            =\left\|\frac{1}{|\mathcal{D}_{\mathcal{S}}|}\sum_{\mathbf{a}\in\mathbf{Z}}\sum_{\mathbf{X}\in{\mathcal{D}_{\mathcal{S}|\mathbf{Z}}}} \mathbf{R}_{\mathbf{a}}
            -\frac{1}{|\mathcal{D}_{\mathcal{T}}|}\sum_{\mathbf{a}\in\mathbf{Z}}\sum_{\mathbf{X}\in\mathcal{D}_{\mathcal{T}|\mathbf{Z}}} \mathbf{R}_{\mathbf{a}}\right\|_{2}^{2}\\
            & d_{\mathrm{CMMD}}^2(\mathcal{S}, \mathcal{S}|\mathbf{Z})\Rightarrow\mathcal{L}_{3}(\mathcal{D}_{\mathcal{S}}, \mathcal{D}_{\mathcal{S}|\mathbf{Z}};\mathcal{G}_{\mathcal{S}}) 
            =\left\|\frac{1}{|\mathcal{D}_{\mathcal{S}}|}\sum_{\mathbf{X}\in{\mathcal{D}_{\mathcal{S}}}} \mathbf{X}
            -\frac{1}{|\mathcal{D}_{\mathcal{S}|\mathbf{Z}}|}\sum_{\mathbf{a}\in\mathbf{Z}}\sum_{\mathbf{X}\in\mathcal{D}_{\mathcal{S}|\mathbf{Z}}} \mathbf{R}_{\mathbf{a}}\right\|_{2}^{2}\\
            & d_{\mathrm{CMMD}}^2(\mathcal{T}, \mathcal{T}|\mathbf{Z})\Rightarrow\mathcal{L}_{4}(\mathcal{D}_{\mathcal{T}}, \mathcal{D}_{\mathcal{T}|\mathbf{Z}};\mathcal{G}_{\mathcal{T}}) 
            =\left\|\frac{1}{|\mathcal{D}_{\mathcal{T}}|}\sum_{\mathbf{X}\in{\mathcal{D}_{\mathcal{T}}}} \mathbf{X}
            -\frac{1}{|\mathcal{D}_{\mathcal{T}|\mathbf{Z}}|}\sum_{\mathbf{a}\in\mathbf{Z}}\sum_{\mathbf{X}\in\mathcal{D}_{\mathcal{T}|\mathbf{Z}}} \mathbf{R}_{\mathbf{a}}\right\|_{2}^{2}
        \end{aligned}
    \end{equation}
    Therefore, we have for Eq.(\ref{eqs:the1})
    \begin{equation}
        \begin{aligned} 
            & \ \ \ \ d_{\mathrm{CMMD}}^2(\mathcal{S}|\mathbf{Z}, \mathcal{T}|\mathbf{Z}) \\
            & \leq \frac{1}{4}\left\|\frac{1}{|\mathcal{D}_{\mathcal{S}}|}\sum_{\mathbf{X}\in{\mathcal{D}_{\mathcal{S}}}} \mathbf{X}
            -\frac{1}{|\mathcal{D}_{\mathcal{T}}|}\sum_{\mathbf{X}\in\mathcal{D}_{\mathcal{T}}} \mathbf{X}\right\|_{2}^{2}
            + \frac{1}{4}\left\|\frac{1}{|\mathcal{D}_{\mathcal{S}}|}\sum_{\mathbf{a}\in\mathbf{Z}}\sum_{\mathbf{X}\in{\mathcal{D}_{\mathcal{S}|\mathbf{Z}}}} \mathbf{R}_{\mathbf{a}}
            -\frac{1}{|\mathcal{D}_{\mathcal{T}}|}\sum_{\mathbf{a}\in\mathbf{Z}}\sum_{\mathbf{X}\in\mathcal{D}_{\mathcal{T}|\mathbf{Z}}} \mathbf{R}_{\mathbf{a}}\right\|_{2}^{2}\\
            & 
             + \frac{1}{4}\left\|\frac{1}{|\mathcal{D}_{\mathcal{S}}|}\sum_{\mathbf{X}\in{\mathcal{D}_{\mathcal{S}}}} \mathbf{X}
            -\frac{1}{|\mathcal{D}_{\mathcal{S}|\mathbf{Z}}|}\sum_{\mathbf{a}\in\mathbf{Z}}\sum_{\mathbf{X}\in\mathcal{D}_{\mathcal{S}|\mathbf{Z}}} \mathbf{R}_{\mathbf{a}}\right\|_{2}^{2}
            + \frac{1}{4}\left\|\frac{1}{|\mathcal{D}_{\mathcal{T}}|}\sum_{\mathbf{X}\in{\mathcal{D}_{\mathcal{T}}}} \mathbf{X}
            -\frac{1}{|\mathcal{D}_{\mathcal{T}|\mathbf{Z}}|}\sum_{\mathbf{a}\in\mathbf{Z}}\sum_{\mathbf{X}\in\mathcal{D}_{\mathcal{T}|\mathbf{Z}}} \mathbf{R}_{\mathbf{a}}\right\|_{2}^{2} \\
            & 
             + \frac{1}{2}\left\|\frac{1}{|\mathcal{D}_{\mathcal{S}}|}\sum_{\mathbf{X}\in{\mathcal{D}_{\mathcal{S}}}} \mathbf{X}
            -\frac{1}{|\mathcal{D}_{\mathcal{T}}|}\sum_{\mathbf{X}\in\mathcal{D}_{\mathcal{T}}} \mathbf{X}\right\|_{2}
            \cdot \left\|\frac{1}{|\mathcal{D}_{\mathcal{S}}|}\sum_{\mathbf{a}\in\mathbf{Z}}\sum_{\mathbf{X}\in{\mathcal{D}_{\mathcal{S}|\mathbf{Z}}}} \mathbf{R}_{\mathbf{a}}
            -\frac{1}{|\mathcal{D}_{\mathcal{T}}|}\sum_{\mathbf{a}\in\mathbf{Z}}\sum_{\mathbf{X}\in\mathcal{D}_{\mathcal{T}|\mathbf{Z}}} \mathbf{R}_{\mathbf{a}}\right\|_{2}
        \end{aligned}
    \end{equation}
    Hence, by utilizing the Eq.(\ref{eqs:reconstruction-x})
    the unified loss function of domain classification $\mathcal{L}(\mathcal{D}_{\mathcal{S}}, \mathcal{D}_{\mathcal{T}},\mathcal{D}_{\mathcal{S}|\mathbf{Z}}, \mathcal{D}_{\mathcal{T}|\mathbf{Z}};\mathcal{G}_{\mathcal{S}}, \mathcal{G}_{\mathcal{T}})$
    can be derived as 
    \begin{equation}
        \begin{aligned}
            & \ \ \ \ \mathcal{L}(\mathcal{D}_{\mathcal{S}}, \mathcal{D}_{\mathcal{T}},\mathcal{D}_{\mathcal{S}|\mathbf{Z}}, \mathcal{D}_{\mathcal{T}|\mathbf{Z}};\mathcal{G}_{\mathcal{S}}, \mathcal{G}_{\mathcal{T}}) \\
            & = \beta_{1}\left\|\frac{1}{|\mathcal{D}_{\mathcal{S}}|}\sum_{\mathbf{X}\in{\mathcal{D}_{\mathcal{S}}}} \mathbf{X}
            -\frac{1}{|\mathcal{D}_{\mathcal{T}}|}\sum_{\mathbf{X}\in\mathcal{D}_{\mathcal{T}}} \mathbf{X}\right\|_{2}^{2}
            + \beta_{2}\left\|\frac{1}{|\mathcal{D}_{\mathcal{S}|\mathbf{Z}}|}\sum_{\mathbf{a}\in\mathbf{Z}}\sum_{\mathbf{X}\in{\mathcal{D}_{\mathcal{S}|\mathbf{Z}}}} \mathbf{R}_{\mathbf{a}}
            -\frac{1}{|\mathcal{D}_{\mathcal{T}|\mathbf{Z}}|}\sum_{\mathbf{a}\in\mathbf{Z}}\sum_{\mathbf{X}\in\mathcal{D}_{\mathcal{T}|\mathbf{Z}}} \mathbf{R}_{\mathbf{a}}\right\|_{2}^{2}\\
            & 
             + \beta_{3}\left\|\frac{1}{|\mathcal{D}_{\mathcal{S}}|}\sum_{\mathbf{X}\in{\mathcal{D}_{\mathcal{S}}}} \mathbf{X}
            -\frac{1}{|\mathcal{D}_{\mathcal{S}|\mathbf{Z}}|}\sum_{\mathbf{a}\in\mathbf{Z}}\sum_{\mathbf{X}\in\mathcal{D}_{\mathcal{S}|\mathbf{Z}}} \mathbf{R}_{\mathbf{a}}\right\|_{2}^{2}
            + \beta_{4}\left\|\frac{1}{|\mathcal{D}_{\mathcal{T}}|}\sum_{\mathbf{X}\in{\mathcal{D}_{\mathcal{T}}}} \mathbf{X}
            -\frac{1}{|\mathcal{D}_{\mathcal{T}|\mathbf{Z}}|}\sum_{\mathbf{a}\in\mathbf{Z}}\sum_{\mathbf{X}\in\mathcal{D}_{\mathcal{T}|\mathbf{Z}}} \mathbf{R}_{\mathbf{a}}\right\|_{2}^{2} \\
            & 
             + \gamma\left\|\frac{1}{|\mathcal{D}_{\mathcal{S}}|}\sum_{\mathbf{X}\in{\mathcal{D}_{\mathcal{S}}}} \mathbf{X}
            -\frac{1}{|\mathcal{D}_{\mathcal{T}}|}\sum_{\mathbf{X}\in\mathcal{D}_{\mathcal{T}}} \mathbf{X}\right\|_{2}
            \cdot \left\|\frac{1}{|\mathcal{D}_{\mathcal{S}|\mathbf{Z}}|}\sum_{\mathbf{a}\in\mathbf{Z}}\sum_{\mathbf{X}\in{\mathcal{D}_{\mathcal{S}|\mathbf{Z}}}} \mathbf{R}_{\mathbf{a}}
            -\frac{1}{|\mathcal{D}_{\mathcal{T}|\mathbf{Z}}|}\sum_{\mathbf{a}\in\mathbf{Z}}\sum_{\mathbf{X}\in\mathcal{D}_{\mathcal{T}|\mathbf{Z}}} \mathbf{R}_{\mathbf{a}}\right\|_{2}
        \end{aligned}
    \end{equation}
    where $\mathbf{R}_{\mathbf{a}}:=\alpha(\mathbf{a}, \mathbf{k})\mathbf{X}$, which $\mathbf{k}$ denote the keys 
    in transfer learning, the constants $\beta$ and $\gamma$ are the balance parameters,
    and $|\cdot|$ is the cardinality function.
\end{proof}

\end{document}